\documentclass{article}

\usepackage{subfigure}
\usepackage{amsmath}
\usepackage{amssymb}
\usepackage{comment}
\usepackage{graphicx}
\usepackage{cancel}
\usepackage{amsthm}
\usepackage{natbib}
\usepackage{dsfont}
\usepackage{xcolor}
\usepackage{ stmaryrd }
\usepackage{enumitem} 

\usepackage{setspace}

\usepackage{booktabs}

\usepackage{times}
\usepackage{graphicx} 
\usepackage{subfigure} 

\usepackage{natbib}

\usepackage[algo2e,inoutnumbered,linesnumbered,algoruled,vlined]{algorithm2e}

\usepackage{hyperref}

\usepackage[accepted]{icml2018}

\newcommand{\x}{{X}}

\newcommand{\stl}{l}
\newcommand{\Lo}{{\cal L}}
\newcommand{\Oc}{{\cal O}}
\newcommand{\Lot}{\tilde{\Lo}}

\newcommand{\D}{ {\cal D} }

\newcommand{\E}{\mathds{E}}

\newtheorem{thm}{Theorem}
\newtheorem{remark}{Remark}

\newtheorem{lemma}{Lemma}
\newtheorem{prop}{Proposition}

\DeclareMathOperator*{\argmin}{arg\min}

\usepackage{cleveref}

\newtheorem{assumption}{\textbf{H}\hspace{-3pt}}
\Crefname{assumption}{\textbf{H}\hspace{-3pt}}{\textbf{H}\hspace{-3pt}}
\crefname{assumption}{\textbf{H}}{\textbf{H}}

\usepackage{multibib}
\newcites{New}{References}

\graphicspath{{./},{./figures}}

\icmltitlerunning{Asynchronous Stochastic Quasi-Newton MCMC}

\begin{document} 

\twocolumn[
\icmltitle{Asynchronous Stochastic Quasi-Newton MCMC for Non-Convex Optimization}




\begin{icmlauthorlist}
\icmlauthor{Umut \c Sim\c sekli}{tpt}
\icmlauthor{\c Ca\u{g}atay Y{\i}ld{\i}z}{aalto}
\icmlauthor{Thanh Huy Nguyen}{tpt}
\icmlauthor{Ga\"{e}l Richard}{tpt} 
\icmlauthor{A. Taylan Cemgil}{boun}
\end{icmlauthorlist}

\icmlaffiliation{aalto}{Department of Computer Science, Aalto University, Espoo, 02150, Finland}
\icmlaffiliation{boun}{Department of Computer Engineering, Bo\u{g}azi\c ci University, 34342, Bebek, Istanbul, Turkey}
\icmlaffiliation{tpt}{LTCI, T\'{e}l\'{e}com ParisTech, Universit\'{e} Paris-Saclay, 75013, Paris, France}

\icmlcorrespondingauthor{Umut \c Sim\c sekli}{umut.simsekli@telecom-paristech.fr}


\icmlkeywords{mcmc}

\vskip 0.3in
]



\printAffiliationsAndNotice{}  

\begin{abstract} 
Recent studies have illustrated that stochastic gradient Markov Chain Monte Carlo techniques have a strong potential in non-convex optimization, where local and global convergence guarantees can be shown under certain conditions. By building up on this recent theory, in this study, we develop an asynchronous-parallel stochastic L-BFGS algorithm for non-convex optimization. The proposed algorithm is suitable for both distributed and shared-memory settings. We provide formal theoretical analysis and show that the proposed method achieves an ergodic convergence rate of ${\cal O}(1/\sqrt{N})$ ($N$ being the total number of iterations) and it can achieve a linear speedup under certain conditions. We perform several experiments on both synthetic and real datasets. The results support our theory and show that the proposed algorithm provides a significant speedup over the recently proposed synchronous distributed L-BFGS algorithm. 
\end{abstract} 

\section{Introduction}
\label{sec:intro}

Quasi-Newton (QN) methods are powerful optimization techniques that are able to attain fast convergence rates by incorporating local geometric information through an approximation of the inverse of the Hessian matrix.
The L-BFGS algorithm \cite{nocedal} is a well-known \emph{limited-memory} QN method that aims at solving the following optimization problem:
\begin{align}
\theta^\star = \argmin_{  \theta \in \mathbb{R}^d} \Bigl\{ U(\theta) \triangleq \sum_{i=1}^{N_Y} U_{i}(\theta) \Bigr\}, \label{eqn:optim_prob}
\end{align}
where $U$ is a twice continuously differentiable function that can be convex or non-convex, and is often referred to as the empirical risk. In a typical machine learning context, a dataset $Y$ with $N_Y$ independent and identically distributed (i.i.d.) data points is considered, which renders the function $U$ as a sum of $N_Y$ different functions $\{U_i\}_{i=1}^{N_Y}$.

In large scale applications, the number of data points $N_Y$ often becomes prohibitively large and therefore using a `batch' L-BFGS algorithm becomes computationally infeasible. As a remedy, \emph{stochastic} L-BGFS methods have been proposed \cite{byrd2016stochastic,schraudolph2007stochastic,moritz2016linearly,zhou2017stochastic,yousefian2017stochastic,zhao2017stochastic}, which aim to reduce the computational requirements of L-BFGS by replacing $\nabla U$ (i.e.\ the full gradients that are required by L-BFGS) with some \emph{stochastic gradients} that are computed on small subsets of the dataset. However, using stochastic gradients within L-BFGS turns out to be a challenging task since it brings additional technical difficulties, which we will detail in Section~\ref{sec:bg}.

In a very recent study, \citet{berahas2016multi} proposed a parallel stochastic L-BFGS algorithm, called multi-batch L-BFGS (mb-L-BFGS), which is suitable for synchronous distributed architectures. This work illustrated that carrying out L-BFGS in a distributed setting introduces further theoretical and practical challenges; however, if these challenges are addressed, stochastic L-BFGS can be powerful in a distributed setting as well, and outperform conventional algorithms such as distributed stochastic gradient descent (SGD), as shown by their experimental results. 

Despite the fact that synchronous parallel algorithms have clear advantages over serial optimization algorithms, the computational efficiency of synchronous algorithms is often limited by the overhead induced by the synchronization and coordination among the worker processes.  
Inspired by asynchronous parallel stochastic optimization techniques \cite{agarwal2011distributed,lian2015asynchronous,zhang2015deep,zhao2015fast,pmlr-v70-zheng17b}, in this study, we propose an asynchronous parallel stochastic L-BFGS algorithm for large-scale non-convex optimization problems. The proposed approach aims at speeding up the synchronous algorithm presented in \cite{berahas2016multi} by allowing all the workers work independently from each other and circumvent the inefficiencies caused by synchronization and coordination.

Extending stochastic L-BFGS to asynchronous settings is a highly non-trivial task and brings several challenges. In our strategy, we first reformulate the optimization problem \eqref{eqn:optim_prob} as a sampling problem where the goal becomes drawing random samples from a distribution whose density is concentrated around $\theta^\star$. We then build our algorithm upon the recent stochastic gradient Markov Chain Monte Carlo (SG-MCMC) techniques \cite{chen2015convergence,chen2016stochastic} that have close connections with stochastic optimization techniques \cite{dalalyan2017further,raginsky17a,zhang17b}, and have proven successful in large-scale Bayesian machine learning. 
We provide formal theoretical analysis and prove non-asymptotic guarantees for the proposed algorithm. Our theoretical results show that the proposed algorithm achieves an ergodic global convergence with rate ${\cal O}(1/\sqrt{N})$, where $N$ denotes the total number of iterations. Our results further imply that the algorithm can achieve a linear speedup under ideal conditions.

For evaluating the proposed method, we conduct several experiments on synthetic and real datasets. The experimental results support our theory: our experiments on a large-scale matrix factorization problem show that the proposed algorithm provides a significant speedup over the synchronous parallel L-BFGS algorithm. 

\section{Technical Background}
\label{sec:bg}

\textbf{Preliminaries: } 
As opposed to the classical optimization perspective, we look at the optimization problem \eqref{eqn:optim_prob} from a \emph{maximum a-posteriori} (MAP) estimation point of view, where we consider $\theta$ as a random variable in $\mathds{R}^d$ and $\theta^\star$ as the optimum of a Bayesian posterior whose density is given as $p(\theta|Y) \propto \exp(-U(\theta))$, where $Y \equiv \{Y_1, \dots, Y_{N_Y}\}$ is a set of i.i.d.\ observed data points. Within this context, $U(\theta)$ is often called the \emph{potential energy} and defined as $U(\theta) =  -[\log p(\theta) + \sum_{i=1}^{N_Y} \log p(Y_i |\theta)]$, where $p(Y_i|\theta)$ is the likelihood function and $p(\theta)$ is the prior density. 
In a classical optimization context, $-\log p(Y_i|\theta)$ would correspond to the data-loss and $-\log p(\theta)$ would correspond to a regularization term. Throughout this study, we will assume that the problem \eqref{eqn:optim_prob} has a unique solution in $\mathds{R}^d$.

We define a stochastic gradient $\nabla \tilde{U}(\theta)$, that is an unbiased estimator of $\nabla U$, as follows:
$\nabla \tilde{U}(\theta) = -[\nabla \log p(\theta) + \frac{N_Y}{N_\Omega} \sum_{i \in \Omega} \nabla \log p(Y_i |\theta)] $,
where $\Omega \subset \{1,\dots,N_Y\}$ denotes a random data subsample that is drawn with replacement, $N_\Omega = |\Omega|$ is the cardinality of $\Omega$. In the sequel, we will occasionally use the notation $\nabla \tilde{U}_n$ and $\nabla \tilde{U}_\Omega$ to denote the stochastic gradient computed at iteration $n$ of a given algorithm, or on a specific data subsample $\Omega$, respectively.

\textbf{The L-BFGS algorithm: }
The L-BFGS algorithm iteratively applies the following equation in order to find the MAP estimate given in \eqref{eqn:optim_prob}: 
\begin{align} 
\theta_n = \theta_{n-1} - h H_n \nabla {U}(\theta_{n-1}) \label{eqn:qn}
\end{align}
where $n$ denotes the iterations. Here, $H_n$ is an approximation to the inverse Hessian at $\theta_{n-1}$ and is computed by using the $M$ past values of the `iterate differences' $s_n \triangleq \theta_{n} - \theta_{n-1}$, and `gradient differences' $y_n \triangleq \nabla {U}(\theta_{n}) - \nabla {U}(\theta_{n-1})$. The collection of the iterate and gradient differences is called the \emph{L-BFGS memory}.
The matrix-vector product $H_n \nabla U(\theta_{n-1})$ is often implemented by using the \emph{two-loop recursion} \cite{nocedal}, which has linear time and space complexities ${\cal O}(Md)$.

In order to achieve computational scalability, stochastic L-BFGS algorithms replace $\nabla U$ with $\nabla \tilde{U}$. This turns out to be problematic, since the gradient differences $y_n$ would be \emph{inconsistent}, meaning that the stochastic gradients in different iterations will be computed on different data subsamples, i.e.\ $\Omega_{n-1}$ and $\Omega_n$. On the other hand, in the presence of the stochastic gradients, L-BFGS is no longer guaranteed to produce positive definite approximations even in convex problems, therefore more considerations should be taken in order to make sure that $H_n$ is positive definite.

\textbf{Stochastic Gradient Markov Chain Monte Carlo: }
Along with the recent advances in MCMC techniques, diffusion-based algorithms have become increasingly popular due to their applicability in large-scale machine learning applications. These techniques, so called the Stochastic Gradient MCMC (SG-MCMC) algorithms, aim at generating samples from the posterior distribution $p(\theta|Y)$ as opposed to finding the MAP estimate, and have strong connections with stochastic optimization techniques \cite{dalalyan2017further}. In this line of work, Stochastic Gradient Langevin Dynamics (SGLD)  \cite{WelTeh2011a} is one of the pioneering algorithms and generates an approximate sample $\theta_n$ from $p(\theta|Y)$ by iteratively applying the following update equation: 
\begin{align} 
\theta_n = \theta_{n-1} - h \nabla \tilde{U}_n(\theta_{n-1}) + \sqrt{2h/\beta} Z_n \label{eqn:sgld}
\end{align}
where $h$ is the step-size and $\{Z_n\}_{n=1}^N$ is a collection of standard Gaussian random variables in $\mathds{R}^d$. 
Here, $\beta$ is called the \emph{inverse temperature}: it is fixed to $\beta = 1$ in vanilla SGLD and when $\beta \neq 1$ the algorithm is called `tempered'.
In an algorithmic sense, SGLD is identical to SGD, except that it injects a Gaussian noise at each iteration and it coincides with SGD when $\beta$ goes to infinity.

SGLD has been extended in several directions \cite{ma2015complete,chen2015convergence,simsekli2016stochastic,simsekli17a}. In \cite{simsekliICML2016}, we proposed an L-BFGS-based SGLD algorithm with ${\cal O}(M^2 d)$ computational complexity, which aimed to improve the convergence speed of the vanilla SGLD. %
We showed that a straightforward way of combining L-BFGS in SGLD would incur an undesired bias; however, the remedy to prevent this bias resulted in numerical instability, which would limit the applicability of the algorithm. 
In other recent studies, SGLD has also been extended to synchronous \cite{ahn2014distributed} and asynchronous \cite{chen2016stochastic,springenberg2016asynchronous} distributed MCMC settings.

SGLD can be seen as a discrete-time simulation of a continuous-time Markov process that is the solution of the following stochastic differential equation (SDE):
\begin{align}
d \theta_t = - \nabla U(\theta_t) dt + \sqrt{2/\beta} dW_t, \label{eqn:langevin}
\end{align}
where $W_t$ denotes the standard Brownian motion in $\mathds{R}^d$.
Under mild regularity conditions on $U$, the solution process $(\theta_t)_{t \geq 0}$ attains a unique stationary distribution with a density that is proportional to $\exp(-\beta U(\theta))$ \cite{Roberts03}. An important property of this distribution is that, as $\beta$ goes to infinity, this density concentrates around the global minimum of $U(\theta)$ \cite{hwang1980laplace,gelfand1991recursive}. Therefore, for large enough $\beta$, a random sample that is drawn for the stationary distribution of  $(\theta_t)_{t \geq 0}$ would be close to $\theta^\star$.
Due to this property, SG-MCMC methods have recently started drawing attention from the non-convex optimization community. \citet{chen2016bridging} developed an annealed SG-MCMC algorithm for non-convex optimization and it was recently extended by \citet{ye_cont_tempering_2017}. \citet{raginsky17a} and \citet{xu2017global} provided finite-time guarantees for SGLD to find an `approximate' global minimizer that is close to $\theta^\star$, which imply that the additive Gaussian noise in SGLD can help the algorithm escape from poor local minima.
In a complementary study, \citet{zhang17b} showed that SGLD enters a neighborhood of a local minimum of $U(\theta)$ in polynomial time, which shows that even if SGLD fails to find the global optimum, it will still find a point that is close to one of the local optima. 
Even though these results showed that SG-MCMC is promising for optimization, it is still not clear how an asynchronous stochastic L-BFGS method could be developed within an SG-MCMC framework.

\section{Asynchronous Stochastic L-BFGS}

In this section, we propose a novel asynchronous L-BFGS-based (tempered) SG-MCMC algorithm that aims to provide an approximate optimum that is close to $\theta^\star$ by generating samples from a distribution that has a density that is proportional to $\exp(-\beta U(\theta))$.
We call the proposed algorithm asynchronous parallel stochastic L-BFGS (as-L-BFGS). Our method is suitable for both distributed and shared-memory settings. We will describe the algorithm only for the distributed setting; the shared-memory version is almost identical to the distributed version as long as the updates are ensured to be \emph{atomic}.

We consider a classical asynchronous optimization architecture, which is composed of a \emph{master node}, several \emph{worker nodes}, and a \emph{data server}. The main task of the master node is to maintain the newest iterate of the algorithm. At each iteration, the master node receives an \emph{additive update vector} from a worker node, it adds this vector to the current iterate in order to obtain the next iterate, and then it sends the new iterate to the worker node which has sent the update vector. On the other hand, the worker nodes work in a completely asynchronous manner. A worker node receives the iterate from the master node, computes an update vector, and sends the update vector to the master node. However, since the iterate would be possibly modified by another worker node which runs asynchronously in the mean time, the update vector that is sent to the server will thus be computed on an \emph{old} iterate, which causes both practical and theoretical challenges. Such updates are aptly called `delayed' or `stale'. The full data is kept in the data server and we assume that all the workers have access to the data server.

The proposed algorithm iteratively applies the following update equations in the master node:
\begin{align}
u_{n+1} = u_n + \Delta u_{n+1}, \qquad
\theta_{n+1} = \theta_{n} + \Delta \theta_{n+1} , \label{eqn:update_th_ult}   
\end{align}
where $n$ is the iteration index, $u_n$ is called the \emph{momentum} variable, and $\Delta u_{n+1}$ and $ \Delta \theta_{n+1}$ are the update vectors that are computed by the worker nodes. A worker node runs the following equations in order to compute the update vectors:
\begin{align}
 \nonumber \Delta u_{n+1} \triangleq &- h' H_{n+1}(\theta_{n-l_n}) \nabla \tilde{U}_{n+1}(\theta_{n-l_n}) - \gamma' u_{n-l_n} \\ 
 &+ \sqrt{2h'\gamma'/\beta  } Z_{n+1},  \label{eqn:delta_u} \\
   \Delta \theta_{n+1} \triangleq &  \>  H_{n+1}(\theta_{n-l_n}) u_{n-l_n}, \label{eqn:delta_th}
\end{align}
where $h'$ is the step-size, $\gamma'>0$ is the \emph{friction} parameter that determines the weight of the momentum, $\beta$ is the inverse temperature, $\{Z_n\}_n$ denotes standard Gaussian random variables, and $H_n$ denotes the L-BFGS matrix at iteration $n$. Here, $l_n \geq 0$ denotes the `staleness' of a particular update and measures the delay between the current update and the up-to-date iterate that is stored in the master node. 
We assume that the delays are bounded, i.e.\ $\max_n l_n \leq l_\text{max} < \infty$. Note that the matrix-vector products have ${\cal O}(Md)$ time-space complexity.

\setlength{\textfloatsep}{5pt}
 \begin{algorithm2e} [t]
 \SetInd{0.000ex}{1.5ex}
 \DontPrintSemicolon
 \SetKwInOut{Input}{input}
 \Input{$\theta_0$, $u_0$}
 {\color{purple} \small \tcp{Global iteration index}}
 $n \leftarrow 0$\\
 Send $(\theta_0,u_0)$ to all the workers $w = 1,\dots,W$  \\
 \While{$n < N$}{
    Receive $(\Delta \theta_{n+1}, \Delta u_{n+1})$ from worker $w$ \vspace{2pt}\\
    {\color{purple} \small \tcp{Generate the new iterates}}
    $u_{n+1} = u_n + \Delta u_{n+1}, \>\>\> \theta_{n+1} = \theta_{n} + \Delta \theta_{n+1}$ \vspace{2pt}\\
    Send the iterates $(\theta_{n+1}, u_{n+1})$ to worker $w$ \vspace{2pt} \\
    Set $n \leftarrow n+1$
 }
 \caption{as-L-BFGS: Master node}
 \label{algo:master}
 \end{algorithm2e}

Due to the asynchrony, the stochastic gradients and the L-BFGS matrices will be computed on the delayed variables $\theta_{n-l_n}$ and $u_{n-l_n}$. 
As opposed to the asynchronous stochastic gradient algorithms, where the main difficulty stems from the delayed gradients, our algorithm faces further challenges since it is not straightforward to obtain the gradient and iterate differences that are required for the L-BFGS computations in an asynchronously parallel setting.

We propose the following approach for the computation of the L-BFGS matrices. As opposed to the mb-L-BFGS algorithm, which uses a central L-BFGS memory (i.e.\ the collection of the gradient and iterate differences) that is stored in the master node, we let each worker have their own local L-BFGS memories since the master node would not be able to keep track of the gradient and iterate differences, which are received in an asynchronous manner. In our strategy, each worker updates its own L-BFGS memory right after sending the update vector to the master node. The overall algorithm is illustrated in Algorithms~\ref{algo:master} and \ref{algo:worker} ($W$ denotes the number of workers).
 
In order to be able to have consistent gradient differences, each worker applies a multi-batch subsampling strategy that is similar to mb-L-BFGS. We divide the data subsample into two subsets, i.e.\ $\Omega_n = \{ S_n, O_n \}$ with $N_S \triangleq |S_n|$, $N_O \triangleq |O_n|$, and $N_\Omega = N_S + N_O $. Here the main idea is to choose $N_S \gg N_O$ and use $O_n$ as an overlapping subset for the gradient differences. In this manner, in addition to the gradients that are computed on $S_n$ and $O_n$, we also perform an extra gradient computation on the previous overlapping subset, at the end of each iteration. As $N_O$ will be small, this extra cost will not be significant.
Finally, in order to ensure the L-BFGS matrices are positive definite, we use a `cautious' update mechanism that is useful for non-convex settings \cite{li2001global,zhang2011quasi,berahas2016multi} as shown in Algorithm~\ref{algo:worker}.

 \begin{algorithm2e} [t]
 \SetInd{0.0ex}{1.5ex}
 \DontPrintSemicolon
 \SetKwInOut{Input}{input}
 \Input{$M$, $\gamma$, $ N_S $, $N_O$ ($N_\Omega = N_S + N_O$)}
 {\color{purple} \small \tcp{Local iteration index}}
 $i \leftarrow 0$ \\ 
 \While{the master node is running}{
    Receive  $(\theta_{n-l_n}, u_{n-l_n})$ from the master \vspace{2pt} \\
    Draw a subsample $\Omega_{n+1} = \{S_{n+1}, O_{n+1}\} $ \vspace{2pt} \\ 
    {\color{purple} \small \tcp{Gradient computation}}
    $\nabla \tilde{U}_{n+1} (\theta_{n-l_n}) = \frac{N_O}{N_\Omega}\nabla \tilde{U}_{O_{n+1}}(\theta_{n-l_n}) + \frac{N_S}{N_\Omega} \nabla \tilde{U}_{S_{n+1}}(\theta_{n-l_n})$ \vspace{4pt}\\
    Compute $(\Delta \theta_{n+1}, \Delta u_{n+1})$ by \eqref{eqn:delta_u} and \eqref{eqn:delta_th} \vspace{2pt} \\ 
    Send $(\Delta \theta_{n+1}, \Delta u_{n+1})$ to the master \vspace{2pt}\\
    {\color{purple} \small \tcp{Local variables for L-BFGS}}
    $\tilde{\theta}_{i} = \theta_{n-l_n}, \>\>\>\> \tilde{O}_{i} = O_{n+1}, \>\>\>\> \tilde{g}_{i} = \nabla \tilde{U}_{O_{n+1}}(\theta_{n-l_n})$\\
    \If{$i \geq 1$}{
    {\color{purple} \small \tcp{Compute the overlapping gradient}} \vspace{2pt}
    $g' = \nabla \tilde{U}_{\tilde{O}_{i-1}}(\tilde{\theta}_{i})$ \\
    {\color{purple} \small \tcp{Compute the L-BFGS variables}}
    $s_{i} = \tilde{\theta}_{i} - \tilde{\theta}_{i-1}, \>\>\> y_{i} = g' - \tilde{g}_{i-1}$ \\
    {\color{purple} \small \tcp{Cautious memory update}}
    Add $(s_i,y_i)$ to the L-BFGS memory only if  $\phantom{ sadhalskaaa} y_i^\top s_i \geq \epsilon \|s_i\|^2$ for some $\epsilon>0$
    }
    Set $i \leftarrow i+1$
 }
 \caption{as-L-BFGS: Worker node ($w$)}
 \label{algo:worker}
 \end{algorithm2e}

Note that, in addition to asynchrony, the proposed algorithm also extends the current stochastic L-BFGS methods by introducing momentum. This brings two critical practical features: (i) without the existence of the momentum variables, the injected Gaussian noise must depend on the L-BFGS matrices, as shown in \cite{simsekliICML2016}, which results in an algorithm with ${\cal O}(M^2d)$ time complexity whereas our algorithm has ${\cal O}(Md)$ time complexity,
(ii) the use of the momentum significantly repairs the numerical instabilities caused by the asynchronous updates, since $u_n$ inherently encapsulates a direction for $\theta_n$, which provides additional information to the algorithm besides the gradients and L-BFGS computations.
Furthermore, in a very recent study \cite{loizou2017momentum} the use of momentum variables has been shown to be useful in other second-order optimization methods.
On the other hand, despite their advantages, the momentum variable also drifts apart the proposed algorithm from the original L-BFGS formulation. However, even such approximate approaches have proven useful in various scenarios \cite{zhang2011quasi,fu2016quasi}. Also note that, when $\beta \rightarrow \infty$, $l_\text{max} = 0$, and $H_n(\theta) = I$ for all $n$, the algorithm coincides with SGD with momentum. A more detailed illustration is given in the supplementary document.

\section{Theoretical Analysis}

In this section, we will provide non-asymptotic guarantees for the proposed algorithm. Our analysis strategy is different from the conventional analysis approaches for stochastic optimization and makes use of tools from analysis of SDEs. 
In particular, we will first develop a continuous-time Markov process whose marginal stationary measure admits a density that is proportional to $\exp(-\beta U(\theta))$. Then we will show that \eqref{eqn:update_th_ult}-\eqref{eqn:delta_th} form an approximate Euler-Maruyama integrator that approximately simulates this continuous process in discrete-time. Finally, we will analyze this approximate numerical scheme and provide a non-asymptotic error bound. All the proofs are given in the supplementary document.

We start by considering the following stochastic dynamical system:
\begin{align}
\nonumber d p_t \hspace{-2pt}&=\hspace{-2pt} \Bigl[\frac1{\beta}\Gamma_{t}(\theta_t) -H_t(\theta_t) \nabla_\theta U(\theta_t) - \gamma p_t   \Bigr] dt \hspace{-2pt} + \hspace{-2pt} \sqrt{\frac{2\gamma}{\beta}  } dW_t\\ 
d \theta_t \hspace{-1pt}&=\hspace{-1pt} H_t(\theta_t) p_t dt 
  \label{eqn:sde}
\end{align}
where $p_t\in \mathds{R}^{d}$ is also called the \emph{momentum} variable, $H_t(\cdot)$ denotes the L-BFGS matrix at time $t$ and $\Gamma_{t}(\cdot)$ is a vector that is defined as follows:
\begin{align}
\Bigl[\Gamma_{t}(\theta)\Bigr]_i \triangleq \sum_{j=1}^d \frac{\partial [H_t(\theta)]_{ij} }{\partial [\theta]_j}, \label{eqn:gamma_term}
\end{align}
where $[v]_i$ denotes the $i^\text{th}$ component of a vector $v$ and similarly $[M]_{ij}$ denotes a single element of a matrix $M$.

In order to analyze the invariant measure of the SDE defined in \eqref{eqn:sde}, we need certain conditions to hold. First,
we have two regularity assumptions on $U$ and $H_t$:
\begin{assumption}
The gradient of the potential is Lipschitz continuous, i.e. $\| \nabla_\theta U (\theta) - \nabla_\theta U (\theta') \| \leq L \| \theta - \theta'\|, \> \forall \theta, \theta' \in \mathds{R}^d$.
\label{asmp:lipschitz}
\end{assumption}
\begin{assumption}
The L-BFGS matrices have bounded second-order derivatives and they are Lipschitz continuous, i.e. $\Vert H_t(\theta) - H_t(\theta')\Vert \leq L_H\Vert\theta - \theta'\Vert, \> \forall \theta, \theta' \in \mathds{R}^d , t\geq 0$.
\label{asmp:H_lipschitz}
\end{assumption}
The assumptions \Cref{asmp:lipschitz} and \Cref{asmp:H_lipschitz} are standard conditions in analysis of SDEs \cite{duan} and similar assumptions have also been considered in stochastic gradient \cite{moulines2011non} and stochastic L-BFGS algorithms \cite{zhou2017stochastic}. Besides, \Cref{asmp:H_lipschitz} provides a direct control on the partial derivatives of $H_t$, which will be useful for analyzing the overall numerical scheme.
We now present our first result that establishes the invariant measure of the SDE \eqref{eqn:sde}. 
\begin{prop}
\label{prop:inv_meas_simple}
Assume that the conditions \Cref{asmp:lipschitz,asmp:H_lipschitz} hold. Let $\x_t = [\theta_t^\top,p_t^\top]^\top \in \mathds{R}^{2d}$ and $(\x_t)_{t\geq 0}$ be a Markov process that is a solution of the SDE given in \eqref{eqn:sde}. Then $(\x_t)_{t\geq 0}$ has a unique invariant measure $\pi$ that admits a density $\rho(\x) \propto \exp(-{\cal E}(\x))$ with respect to the Lebesgue measure, where ${\cal E}$ is an energy function on the extended state space and is defined as:
${\cal E}(\x) \triangleq \beta U(\theta) + \frac{\beta}{2} p^\top p$.
\end{prop}
This result shows that, if the SDE \eqref{eqn:sde} could be exactly simulated, the \emph{marginal} distribution of the samples $\theta_t$ would converge to a measure $\pi_\theta$ which has a density that is proportional to $\exp(-\beta U(\theta))$. Therefore, for large enough $\beta$ and $t$, $\theta_t$ would be close to the global optimum $\theta^\star$.

We note that when $\beta = 1$, the SDE \eqref{eqn:sde} shares similarities with the SDEs presented in \citep{fu2016quasi, ma2015complete}. While the main difference being the usage of the tempering scheme, \cite{fu2016quasi} further differs from our approach as it directly discard the term $\Gamma_t$ since is in a Metropolis-Hastings framework, which is not adequate for large-scale applications. On the other hand, the stochastic gradient Riemannian Hamiltonian Monte Carlo algorithm given in \citep{ma2015complete}, chooses $H_t$ as the Fisher information matrix; a quantity that requires ${\cal O}(d^2)$ space-time complexity and is not analytically available in general.

We will now show that the proposed algorithm \eqref{eqn:update_th_ult}-\eqref{eqn:delta_th} form an approximate method for simulating \eqref{eqn:sde} in discrete-time. For illustration, we first consider the Euler-Maruyama integrator for \eqref{eqn:sde}, given as follows: 
\begin{align}
\nonumber p_{n+1} =& \> p_n -h H_n(\theta_n) \nabla_\theta U(\theta_n) - h \gamma p_n + \frac{h}{\beta} \Gamma_{n}(\theta_n) \\
  &+ \sqrt{2h\gamma/\beta  } Z_{n+1}, \label{eqn:euler_simple} \\
  \theta_{n+1} =& \> \theta_{n} + h H_n(\theta_n) p_n. \label{eqn:euler_th1} 
\end{align}
Here, the term $(1/\beta) \Gamma_n$ introduces an additional computational burden and its importance is very insignificant (i.e.\ its magnitude is of order ${\cal O}(1/\beta)$ due to \Cref{asmp:H_lipschitz}). Therefore, we discard $\Gamma_n$, define $u_n \triangleq h p_n $, $\gamma' \triangleq h\gamma$, $h' \triangleq h^2$, and use these quantities in \eqref{eqn:euler_simple} and \eqref{eqn:euler_th1}. We then obtain the following re-parametrized Euler integrator:
\begin{align*}
 u_{n+1} \hspace{-3pt}=\hspace{-0pt} & u_n \hspace{-2pt}-\hspace{-2pt} h' H_n(\theta_{n}) \nabla_\theta U(\theta_{n}) \hspace{-2pt}-\hspace{-2pt} \gamma' u_{n} \hspace{-2pt}+\hspace{-2pt} \sqrt{2h'\gamma'/\beta  } Z_{n+1}\\ 
 \theta_{n+1} \hspace{-3pt}=&\hspace{-0pt} \theta_{n} \hspace{-2pt}+\hspace{-2pt} H_n(\theta_n) u_{n}
 \end{align*} 
The detailed derivation is given in the supplementary document. Finally, we replace $\nabla U$ with the stochastic gradients, replace the variables $\theta_n$ and $u_n$ with stale variables $\theta_{n-l_n}$ and $p_{n-l_n}$ in the update vectors, and obtain the ultimate update equations, given in \eqref{eqn:update_th_ult}. 
Note that, due to the negligence of $\Gamma_n$, the proposed approach would require a large $\beta$ and would not be suitable for classical posterior sampling settings, where $\beta =1$.

In this section, we will analyze the \emph{ergodic} error $\mathds{E}[\hat{U}_N - U^\star]$, where we define $\hat{U}_N \triangleq (1/N) \sum_{n=1}^N U(\theta_n)$ and $U^\star \triangleq U(\theta^\star)$. This error resembles the bias of a statistical estimator; however, as opposed to the bias, it directly measures the expected discrepancy to the global optimum. Similar ergodic error notions have been considered in the analysis of non-convex optimization methods \cite{lian2015asynchronous,chen2016bridging,berahas2016multi}.

In our proof strategy, we decompose the error into two terms: $\mathds{E}[\hat{U}_N - U^\star] = {\cal A}_1+ {\cal A}_2$, where ${\cal A}_1 \triangleq \mathds{E}[\hat{U}_N - \bar{U}_\beta]$  ${\cal A}_2 \triangleq [\bar{U}_\beta - U^\star] \geq 0$, and $\bar{U}_{\beta} \triangleq \int_{\mathds{R}^d} U(\theta) \pi_{\theta}(d\theta) $. We then upper-bound these terms separately.

The term ${\cal A}_1$ turns out to be the bias of a statistical estimator, which we can analyze by using ideas from recent SG-MCMC studies. However, existing tools cannot be directly used because of the additional difficulties introduced by the L-BFGS matrices. 
In order to bound ${\cal A}_1$, we first require the following smoothness and boundedness condition.
\begin{assumption}
Let $\psi$ be a functional that is the unique solution of a Poisson equation that is defined as follows:
\begin{align}
\Lo_n \psi (\x_n) = U(\theta_n) - \bar{U}_\beta, \label{eqn:poisson_eq}
\end{align}
where $\x_n = [\theta_n^\top,p_n^\top]^\top$, $\Lo_n$ is the generator of \eqref{eqn:sde} at $t=nh$ and is formally defined in the supplementary document.
The functional $\psi$ and its up to third-order derivatives $\D^k \psi$ are bounded by a function $V(\x)$, such that $\|\D^k \psi\| \leq C_k V^{r_k}$ for $k = 0,1,2,3$ and $C_k, r_k>0$. Furthermore, $\sup_n \mathds{E} V^r(\x_n) < \infty$ and $V$ is smooth such that $\sup_{s \in (0,1)} V^r(s\x + (1-s) \x') \leq C(V^r(\x) + V^r(\x'))$ for all $\x,\x' \in \mathds{R}^{2d}$, $r \leq \max{2r_k}$, and $C  >0$.
\label{asmp:poisson}
\end{assumption}
Assumption \Cref{asmp:poisson} is also standard in SDE analysis and SG-MCMC \cite{mattingly2010convergence,teh2016consistency,chen2015convergence,durmus2016stochastic} and gives us control over the weak error of the numerical integrator. We further require the following regularity conditions in order to have control over the error induced by the delayed stochastic gradients.
\begin{assumption}
The variance of the stochastic gradients is bounded, i.e. $\mathds{E} \| \nabla_\theta U(\theta) - \nabla_\theta \tilde{U}(\theta) \|^2 \leq \sigma$ for some $0 < \sigma <\infty $.
\label{asmp:vargrad}
\end{assumption}
\begin{assumption} 
\label{asmp:Taylor expansion}
For a smooth and bounded function $f$, the remainder $r_{\Lo_n,f}(\cdot)$ in the following Taylor expansion is bounded:
\begin{align}
e^{h\Lo_n}f(\x) = f(\x) + h\Lo_n f(\x) + h^2 r_{\Lo_n,f}(\x).
\end{align}
\end{assumption}
The following lemma presents an upper-bound for ${\cal A}_1$.
\begin{lemma}
\label{lem:euler}
Assume the conditions \Cref{asmp:lipschitz}-\ref{asmp:Taylor expansion} hold. We have the following bound for the bias:
\begin{align}
\bigl|\mathds{E}[\hat{U}_N - \bar{U}_\beta] \bigr| = \Oc \Bigl( \frac{1}{Nh} + \max(\stl_{\text{max}},1)h + \frac{1}{\beta}\Bigr). \label{eqn:bound_bias}
\end{align}
\end{lemma}
Here, the term $1/\beta$ in \eqref{eqn:bound_bias} appears due to the negligence of $\Gamma_n$.
In order to bound the second term ${\cal A}_2$, we follow a similar strategy to \cite{raginsky17a}, where we use \Cref{asmp:lipschitz} and the following moment condition on $\pi_\theta$.
\begin{assumption}
The second-order moments of $\pi_\theta$ are bounded and satisfies the following inequality:
$\int_{\mathds{R}^{d}} \|\theta\|^2 \pi_\theta (d \theta) \leq  \frac{C_\beta}{\beta},$
for some $C_\beta > \max ( \beta d/(2 \pi e), d e/L )$. 
\label{asmp:var}
\end{assumption}
This assumption is mild since $\pi_\theta$ concentrates around $\theta^\star$ as $\beta$ tends to infinity. The order $1/\beta$ is arbitrary, hence the assumption can be further relaxed. The following lemma establishes an upper-bound for ${\cal A}_2$. 
\begin{lemma}
\label{lem:entropy}
Under assumptions \Cref{asmp:lipschitz} and \ref{asmp:var}, the following bound holds: 
$\bar{U}_\beta - U^\star = {\cal O}(1/\beta)$.
\end{lemma}
We now present our main result, which can be easily proven by combining Lemmas~\ref{lem:euler} and \ref{lem:entropy}.
\begin{thm}
Assume that the conditions \Cref{asmp:lipschitz}-\ref{asmp:var} hold. Then the ergodic error of the proposed algorithm is bounded as follows:
\begin{align}
\bigl| \mathds{E}\hat{U}_N - U^\star \bigr| = {\cal O}\Bigl( \frac{1}{Nh} + \max(1,\stl_{\text{max}})h + \frac{1}{\beta} \Bigr) . \label{eqn:thm} 
\end{align}
\label{thm:bias}
\vspace{-12pt}
\end{thm}
More explicit constants and a discussion on the relation of the theorem to other recent theoretical results are provided in the supplementary document. 

Theorem~\ref{thm:bias} provides a non-asymptotic guarantee for convergence to a point that is close to the global optimizer $\theta^\star$ even when $U$ is non-convex, thanks to the additive Gaussian noise. The bound suggests an optimal rate of convergence of ${\cal O}(1/\sqrt{N})$, which is in line with the current rates of the non-convex asynchronous algorithms \cite{lian2015asynchronous}.  
Furthermore, if we assume that the total number of iterations $N$ is a linear function of the number of workers, e.g.\ $N = N_W W $, where $N_W$ is the number of iterations executed by a single worker, Theorem~\ref{thm:bias} implies that, in the ideal case, the proposed algorithm can achieve a linear speedup with increasing $W$, provided that $\stl_{\text{max}} = {\cal O}(1/(N h^2))$.

Despite their nice theoretical properties, it is well-known that tempered sampling approaches also often get stuck near a local minimum. In our case, this behavior would be mainly due to the hidden constant in \eqref{eqn:bound_bias}, which can be exponential in dimension $d$, as illustrated in \cite{raginsky17a} for SGLD. On the other hand, Theorem~\ref{thm:bias} does not guarantee that the proposed algorithm will converge to a neighborhood of a local minimum; however, we believe that we can also prove local convergence guarantees by using the techniques provided in \cite{zhang17b,tzen2018local}, which we leave as a future work.

\begin{figure}[t]
\centering
\includegraphics[width=\columnwidth]{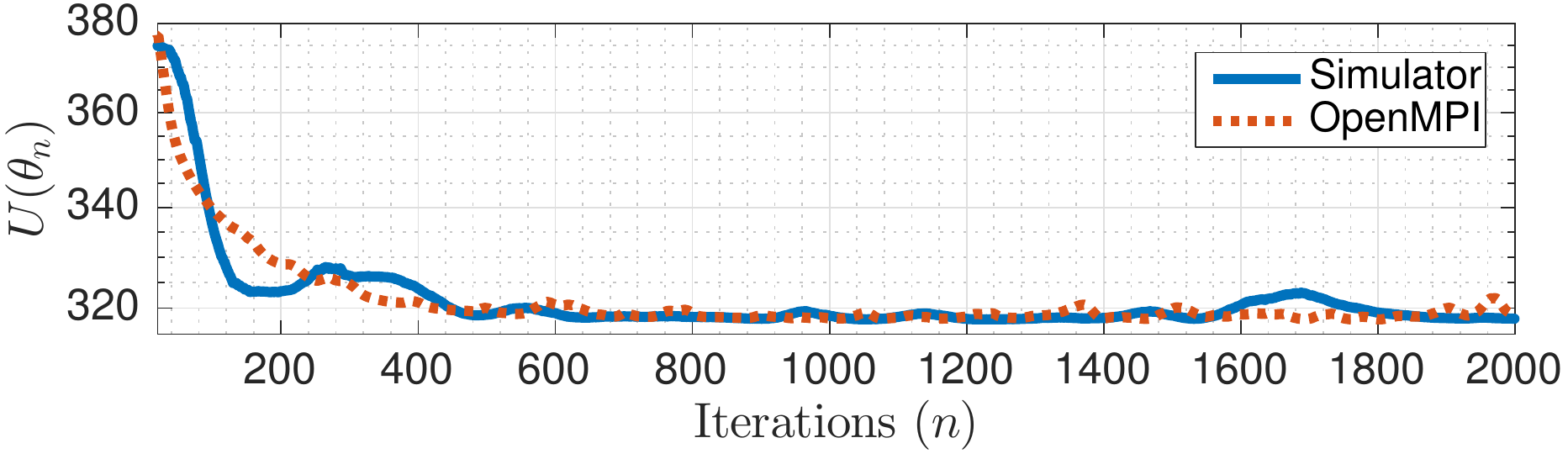}  
\vspace{-20pt}
\caption{The comparison of the simulated and the real implementation of as-L-BFGS with $W=40$ workers.}
\label{fig:simulvscpp}
\vspace{5pt}
\end{figure}

\section{Experiments}

The performance of asynchronous stochastic gradient methods has been evaluated in several studies, where the advantages and limitations have been illustrated in various scenarios, to name a few \cite{dean2012large,zhang2015deep,pmlr-v70-zheng17b}. In this study, we will explore the advantages of using L-BFGS in an asynchronous environment. In order to illustrate the advantages of asynchrony, we will compare as-L-BFGS with mb-L-BFGS \cite{berahas2016multi}; and in order to illustrate the advantages that are brought by using higher-order geometric information, we will compare as-L-BFGS to asynchronous SGD (a-SGD) \cite{lian2015asynchronous}. We will also explore the speedup behavior of as-L-BFGS for increasing $W$.

We conduct experiments on both synthetic and real datasets. For real data experiments, we have implemented all the three algorithms in C++ by using a low-level message passing protocol for parallelism, namely the OpenMPI library. This code can be used both in a distributed environment or a single computer with multiprocessors. For the experiments on synthetic data, we have implemented the algorithms in MATLAB, by developing a realistic \emph{discrete-event simulator}. This simulated environment is particularly useful for understanding the behaviors of the algorithms in detail since we can explicitly control the computation time that is spent at the master or worker nodes, and the communication time between the nodes. This simulation strategy also enables us to explicitly control the variation among the computational powers of the worker nodes; a feature that is much harder to control in real distributed environments. 

\begin{figure}[t]
\centering
\includegraphics[width=\columnwidth]{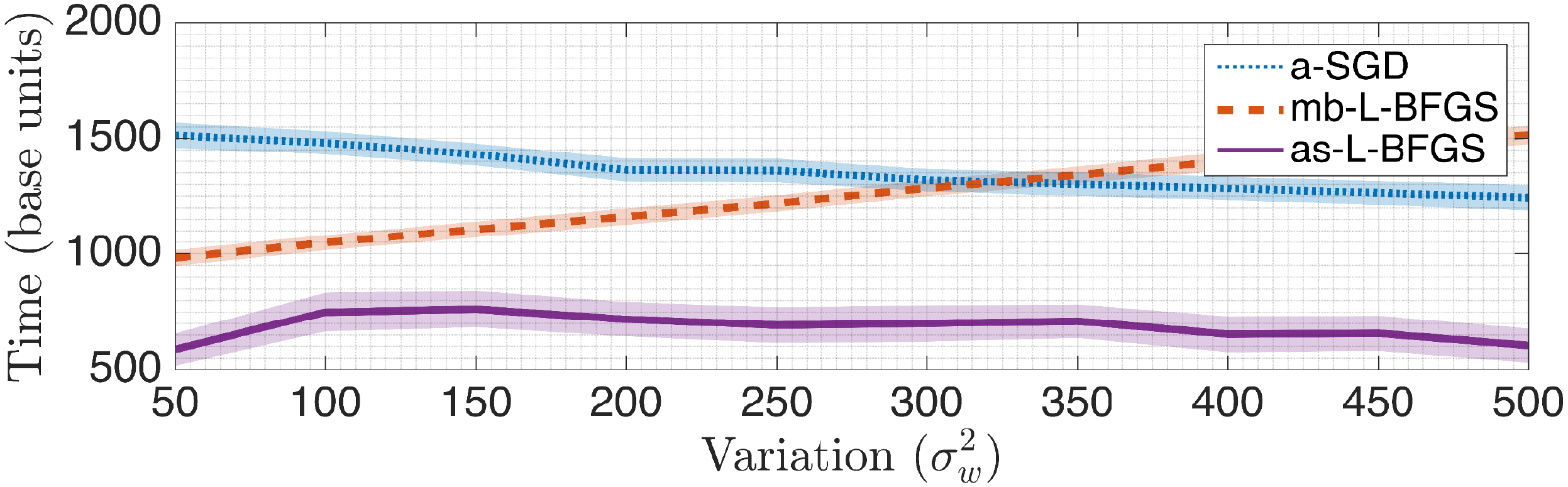}
\caption{The required time to achieve $\varepsilon$-accuracy in the synthetic setting. Solid lines represent the average results and the shades represent three standard deviations.}
\label{fig:simul_res}
\vspace{5pt}
\end{figure}

\textbf{Linear Gaussian model: }
We conduct our first set of experiments on synthetic data where we consider a rather simple convex quadratic problem whose optimum is analytically available. The problem is formulated as finding the MAP estimate of the following linear Gaussian probabilistic model:
\begin{align*}
\theta \sim {\cal N}(0, I), \>\>\> Y_i|\theta \sim {\cal N}(a_i^\top \theta, \sigma_x^2), \>\>\> \forall i =1, \dots, N_Y. 
\end{align*}
We assume that $\{a_n\}_{n=1}^N$ and $\sigma_x^2$ are known and we aim at computing $\theta^\star$. For these experiments, we develop a parametric discrete event simulator that aims to simulate the algorithms in a controllable yet realistic way. The simulator simulates a distributed optimization algorithm once it is provided four parameters: (i) $\mu_m$: the average computational time spent by the master node at each iteration, (ii) $\mu_w$: the average computational time spent by a single worker at each iteration, (iii) $\sigma_w$: the standard deviation of the computational time spent by a single worker per iteration, and (iv) $\tau$: the time spent for communications per iteration. All these parameters are in a generic \emph{base time unit}. Once these parameters are provided for one of the three algorithms, the simulator simulates the (a)synchronous distributed algorithm by drawing random computation times from a log-normal distribution whose mean and variance is specified by $\mu_w$ and $\sigma^2_w$. Figure~\ref{fig:simulvscpp} illustrates a typical outcome of the real and the simulated implementations of as-L-BFGS, where we observe that the simulator is able to provide realistic simulations that can even very well reflect the fluctuations of the algorithm.

In our first experiment, we set $d = 100$, $\sigma_x^2 = 10$, $N_Y=600$, we randomly generate and fix the vectors $\{a_n\}_n$ in such a way that there will be a strong correlation in the posterior distribution, and we finally generate a true $\theta$ and the observations $Y$ by using the generative model. 

For each algorithm, we fix $\mu_m$, $\mu_s$, and $\tau_s$ to realistic values and investigate the effect of the variation among the workers by comparing the running time of the algorithms for achieving $\varepsilon$-accuracy (i.e., $(U(\theta_n)-U^\star)/U^\star \leq \varepsilon$) for different values of $\sigma_w^2$ when $W = 40$. We repeat each experiment $100$ times.
In all our experiments, we have tried several values for the hyper-parameters of each algorithm and we report the best results. All the hyper-parameters are provided in the supplementary document.  

Figure~\ref{fig:simul_res} visualizes the results for the first experiment. We can observe that, for smaller values $\sigma_w^2$ as-L-BFGS and mb-L-BFGS perform similarly, where a-SGD requires more computational time to achieve $\varepsilon$-accuracy. However, as we increase the value of $\sigma_w^2$, mb-L-BFGS requires more computational time in order to be able to collect sufficient amount of stochastic gradients. The results show that both asynchronous algorithms turn out to be more robust to the variability of the computational power of the workers, where as-L-BFGS shows a better performance. 

\begin{figure}[t]
\centering
\includegraphics[width=0.49\columnwidth]{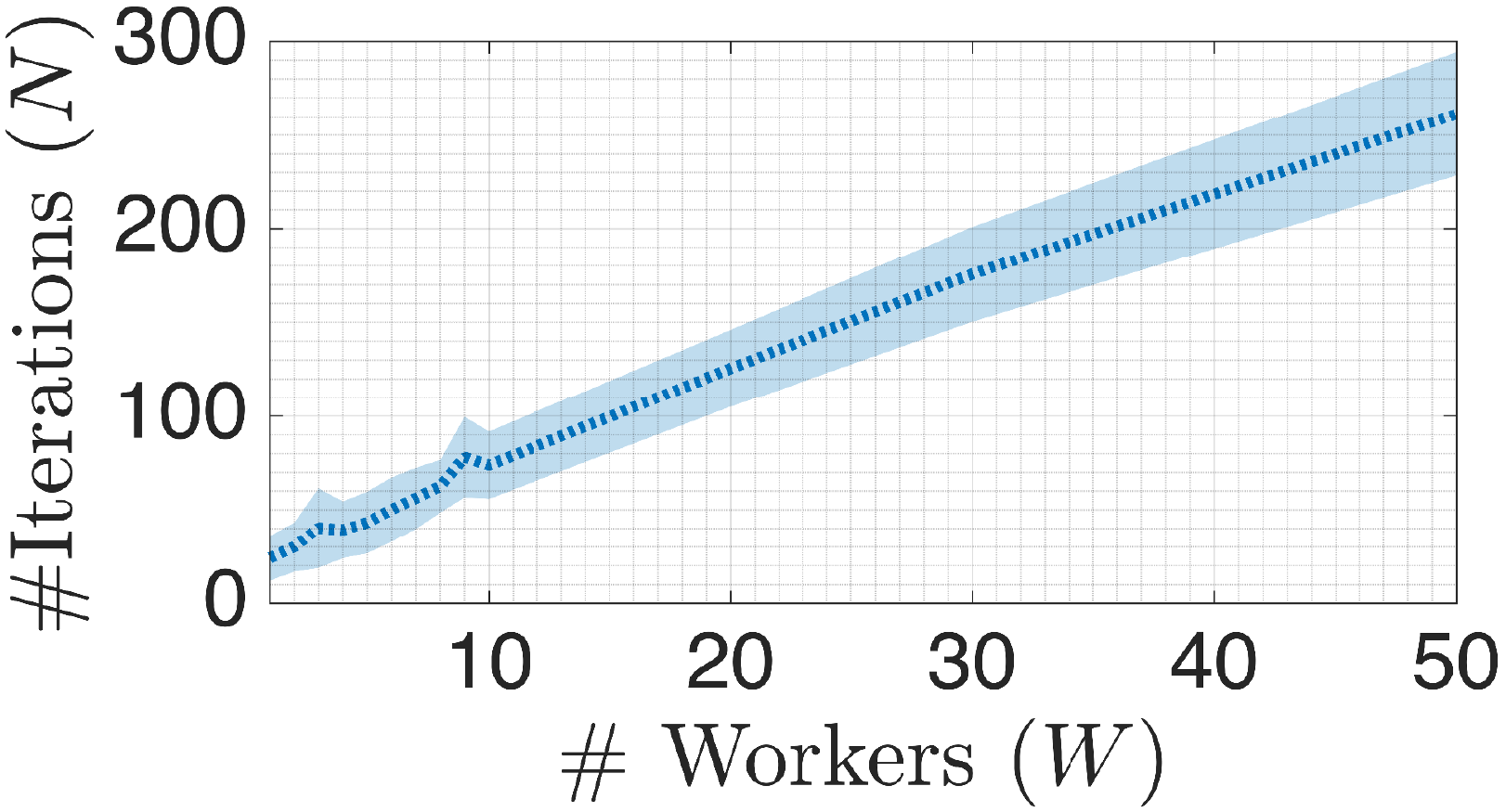}
\includegraphics[width=0.49\columnwidth]{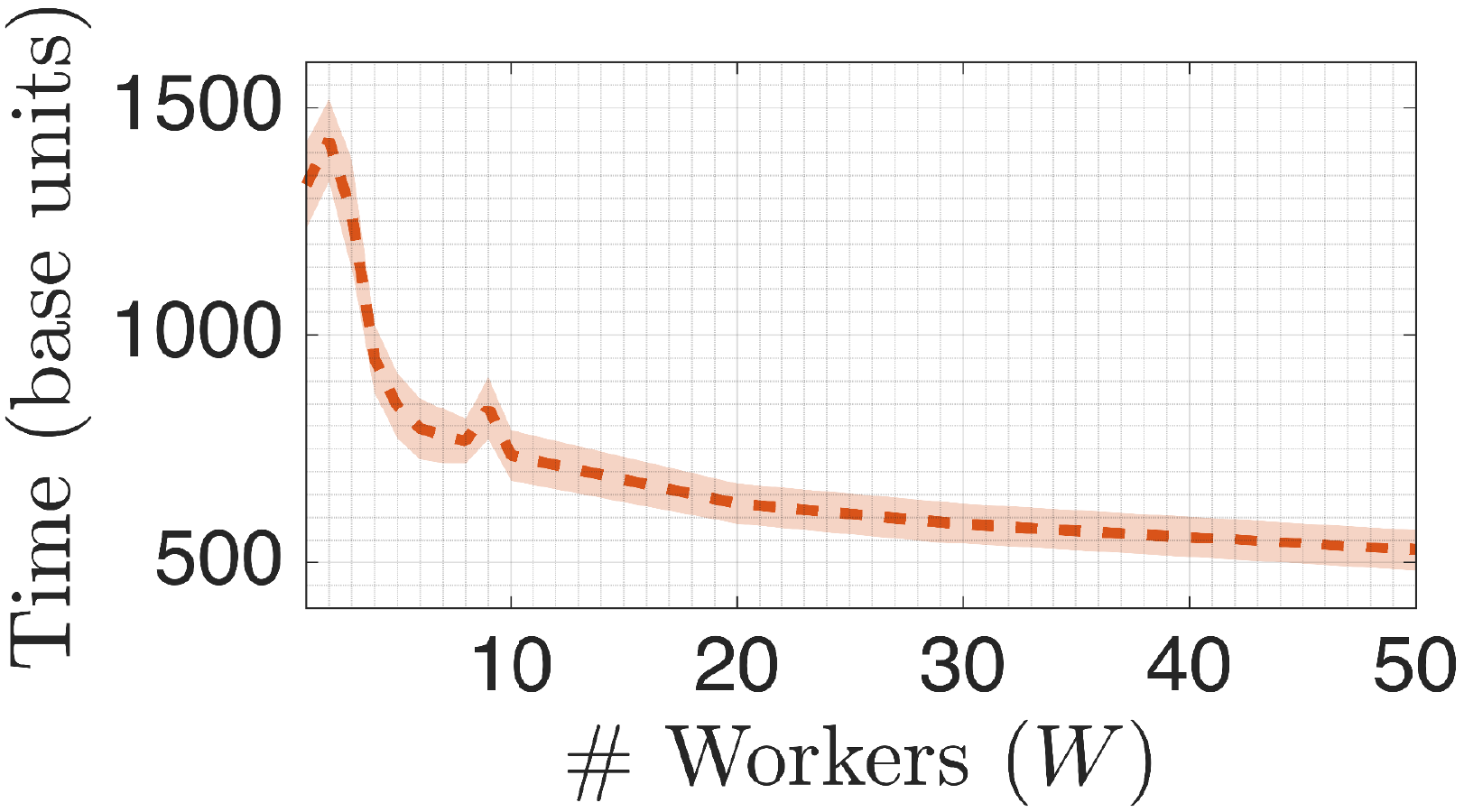}   
\caption{The required time to achieve $\varepsilon$-accuracy by as-L-BFGS for increasing number of workers. Solid lines represent the average results and the shades represent three standard deviations.}
\label{fig:simulspeedup}
\vspace{0pt}
\end{figure}

In our second experiment, we investigate the speedup behavior of as-L-BFGS within the simulated setting. In this setting, we consider a highly varying set of workers and set $\sigma_w^2 = 200$ and vary the number of workers $W$. As illustrated in Figure~\ref{fig:simulspeedup}, as $W$ increases, $\l_\text{max}$ increases as well and the algorithm hence requires more iterations in order to achieve  $\varepsilon$-accuracy, since a smaller step-size needs to be used. However, this increment in the number of iterations is compensated by the increased number of workers, as we observe that the required computational time gracefully decreases with increasing $W$. We observe a similar behavior for different values of $\sigma^2_w$, where the speedup is more prominent for smaller $\sigma_w^2$. 

\textbf{Large-scale matrix factorization: }
In our next set of experiments, we consider a large-scale matrix factorization problem \cite{gemulla2011,csimcsekli2015parallel,simsekli2017parallelized}, where the goal is to obtain the MAP solution of the following probabilistic model: $F_{rk} \sim {\cal N}(0,1), \> G_{ks} \sim {\cal N}(0,1), \> 
Y_{rs} | F,G \sim {\cal N}\bigl(\sum\nolimits_k F_{rk} G_{ks}, 1 \bigr)
$. Here, $Y \in \mathds{R}^{R \times S}$ is the data matrix, and $F \in \mathds{R}^{R \times K}$ and $G \in \mathds{R}^{K \times S}$ are the factor matrices to be estimated.

In this context, we evaluate the algorithms on three large-scale movie ratings datasets, namely MovieLens $1$Million (ML-$1$M), $10$Million (ML-$10$M), and $20$Million (ML-$20$M) (\url{grouplens.org}). 
The ML-$1$M dataset contains $1$ million non-zero entries, where $R = 3883$ (movies) and $S = 6040$ (users). The ML-$10$M dataset contains $10$ million non-zero entries, resulting in a $10681 \times 71567$ data matrix. Finally, the ML-$20$M dataset contains $20$ million ratings, resulting in a $27278 \times 138493$ data matrix.
We have conducted these experiments on a cluster of more than $500$ interconnected computers, each of which is equipped with variable quality CPUs and memories. In these experiments, we have found that the numerical stability is improved when $H_n$ is replaced with $(H_n+ \rho I)$ for small $\rho>0$. This small modification does not violate our theoretical results. The hyper-parameters are provided in the supplementary document.

\begin{figure}[t]
\centering
\includegraphics[width=0.49\columnwidth]{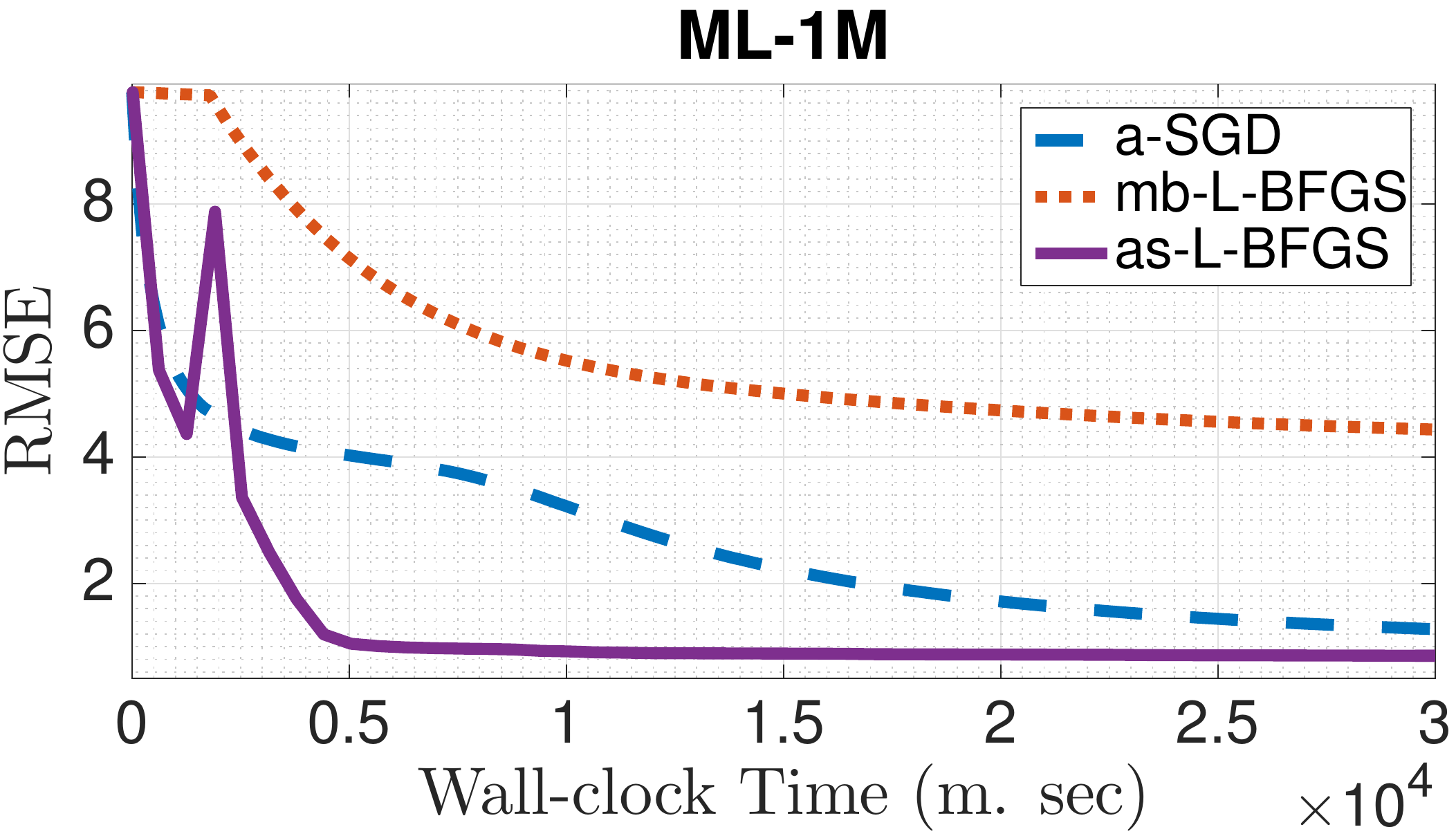}
\includegraphics[width=0.49\columnwidth]{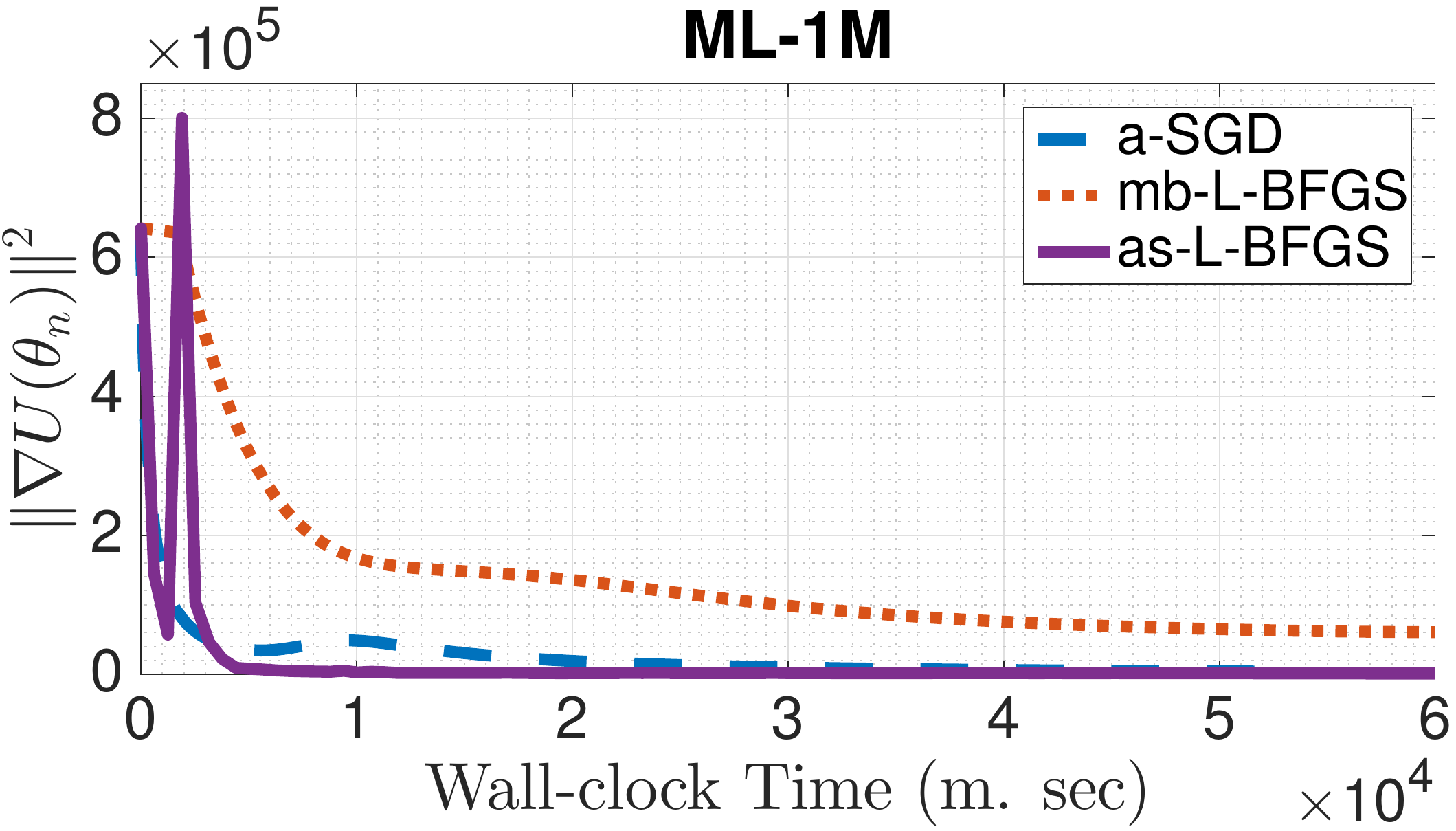}\\
\includegraphics[width=0.49\columnwidth]{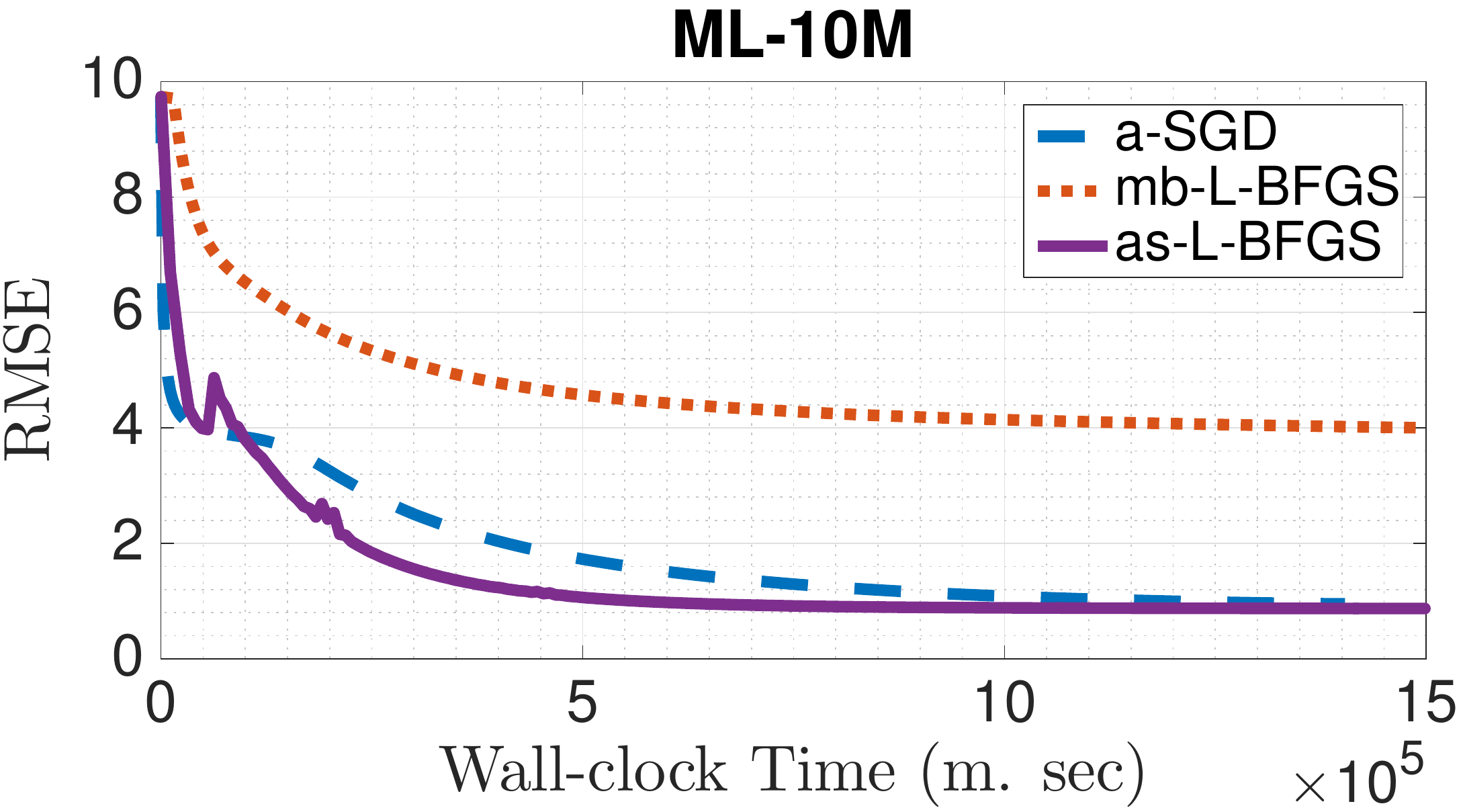}
\includegraphics[width=0.49\columnwidth]{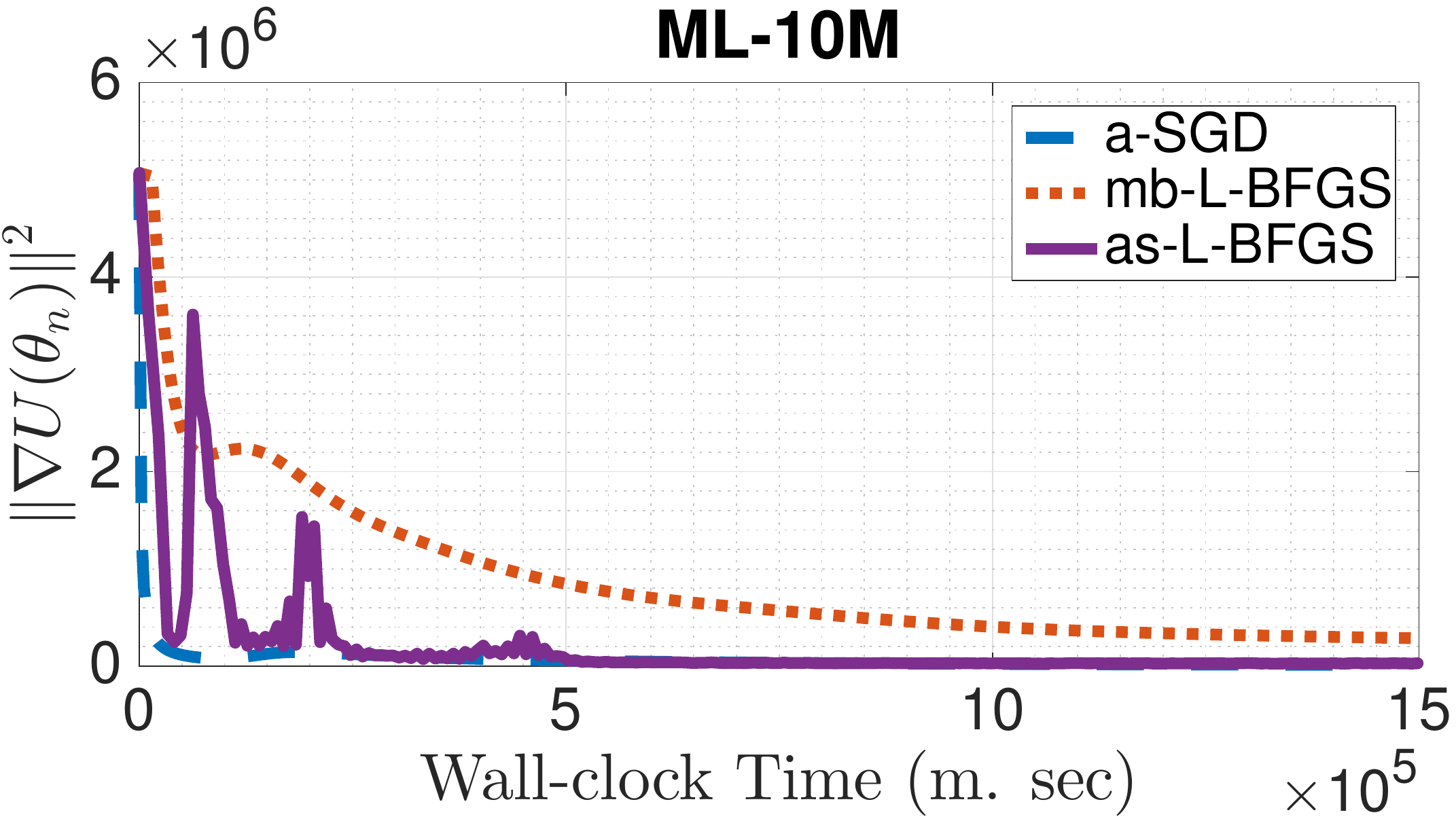} \\
\includegraphics[width=0.49\columnwidth]{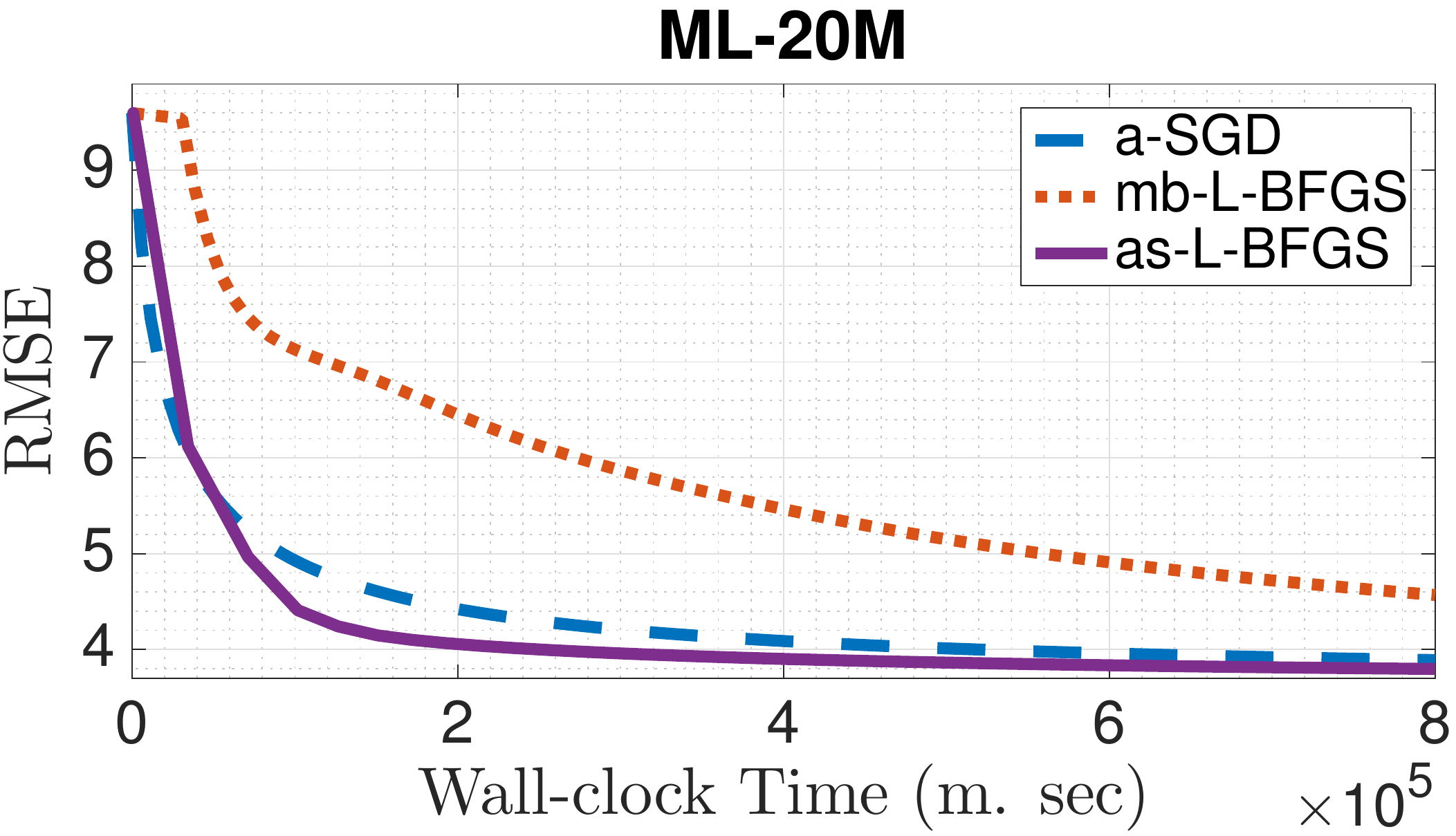}
\includegraphics[width=0.49\columnwidth]{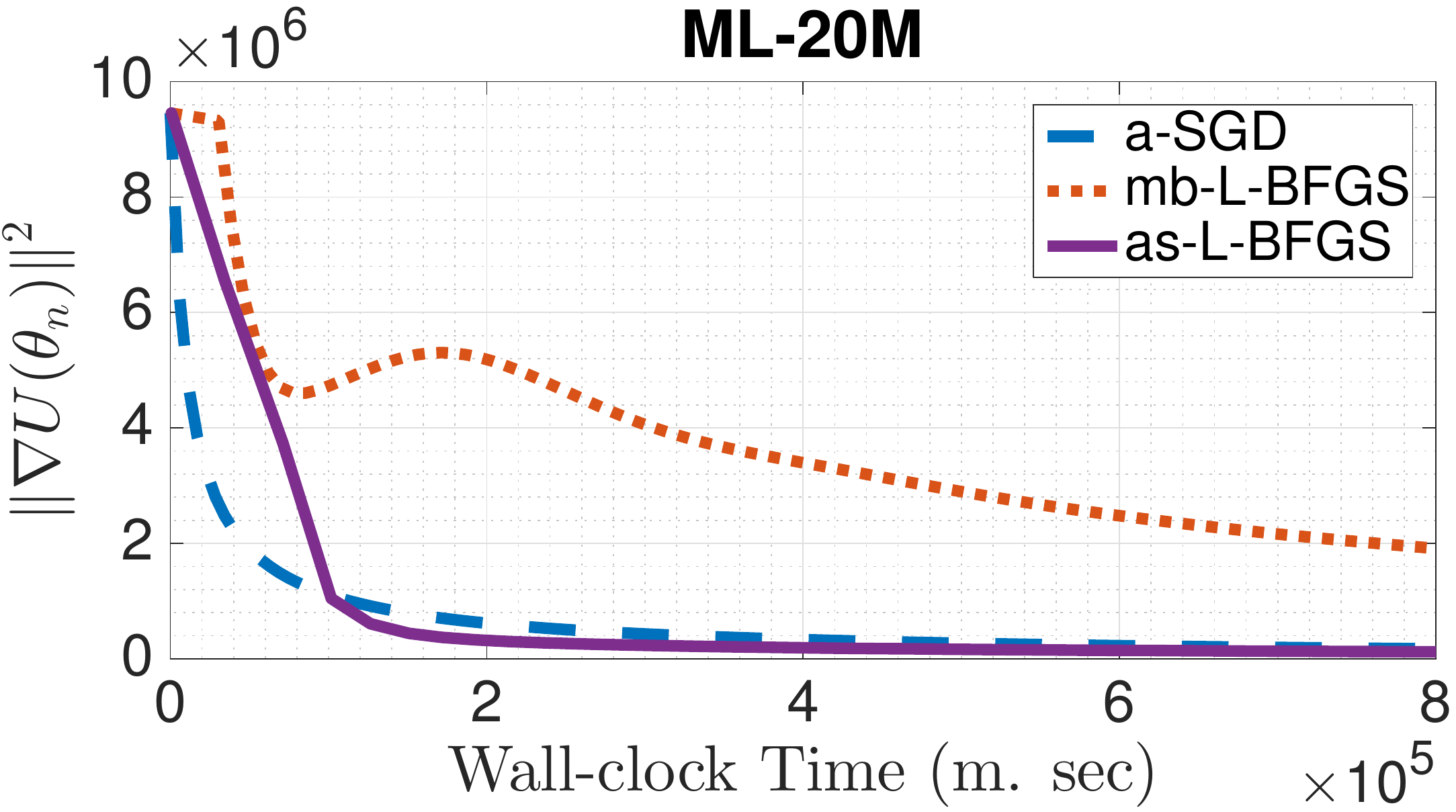}
\caption{The convergence behavior of the algorithms on the MovieLens datasets for $W=10$.  }
\label{fig:ml_exps}
\vspace{8pt}
\end{figure}

Figure~\ref{fig:ml_exps} shows the performance of the three algorithms on the MovieLens datasets in terms of the root-mean-squared-error (RMSE), which is a standard metric for recommendation systems, and the norm of the gradients through iterations. In these experiments, we set $K=5$ for all the three datasets and we set the number of workers to $W=10$. The results show that, in all datasets, as-L-BFGS provides a significant speedup over mb-L-BFGS thanks to asynchrony. We can observe that even when the speed of convergence of mb-L-BFGS is comparable to a-SGD and as-L-BFGS (cf.\ the plots showing the norm of the gradients), the final RMSE yielded by mb-L-BFGS is poorer than the two other methods, which is an indicator that the asynchronous algorithms are able to find a better local minimum. On the other hand, the asynchrony causes more fluctuations in as-L-BFGS when compared to a-SGD.

As opposed to the synthetic data experiments, in all the three MovieLens datasets, we observe that as-L-BFGS provides a slight improvement in the convergence speed when compared to a-SGD. This indicates that a-SGD is able to achieve a comparable convergence speed by taking more steps while as-L-BFGS is computing the matrix-vector products. However, this gap can be made larger by considering a more efficient, yet more sophisticated implementation for L-BFGS computations \cite{chen2014large}.

In our last experiment, we investigate the speedup properties of as-L-BFGS in the real distributed setting. In this experiment, we only consider the ML-$1$M dataset and run the as-L-BFGS algorithm for different number of workers. Figure~\ref{fig:ml_speedup} illustrates the results of this experiment. As we increase $W$ from $1$ to $10$, we obtain a decent speedup that is close to a linear speedup. However, when we set $W=20$ the algorithm becomes unstable, since the term $\stl_\text{max}h$ in \eqref{eqn:thm} dominates. Therefore, for $W=20$ we need to decrease the step-size $h$, which requires the algorithm to be run for a longer amount of time in order to achieve the same error as we achieved when $W$ was smaller. On the other hand, the algorithm achieves a linear speedup in terms of iterations; however, the corresponding result is provided in the supplementary document due to the space constraints.

\begin{figure}[t]
\centering
\includegraphics[width=0.52\columnwidth]{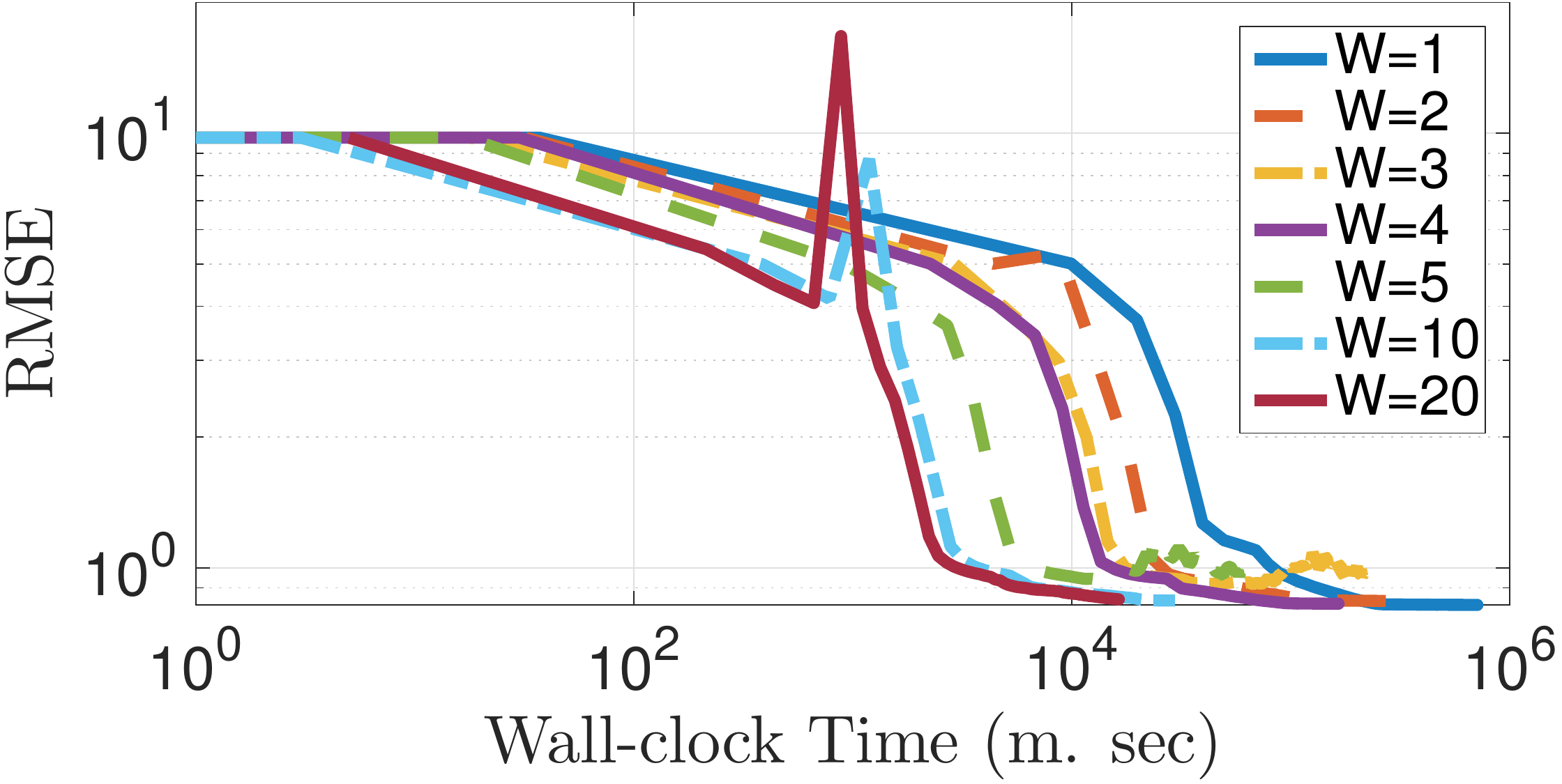} \hspace{10pt}
\includegraphics[width=0.31\columnwidth]{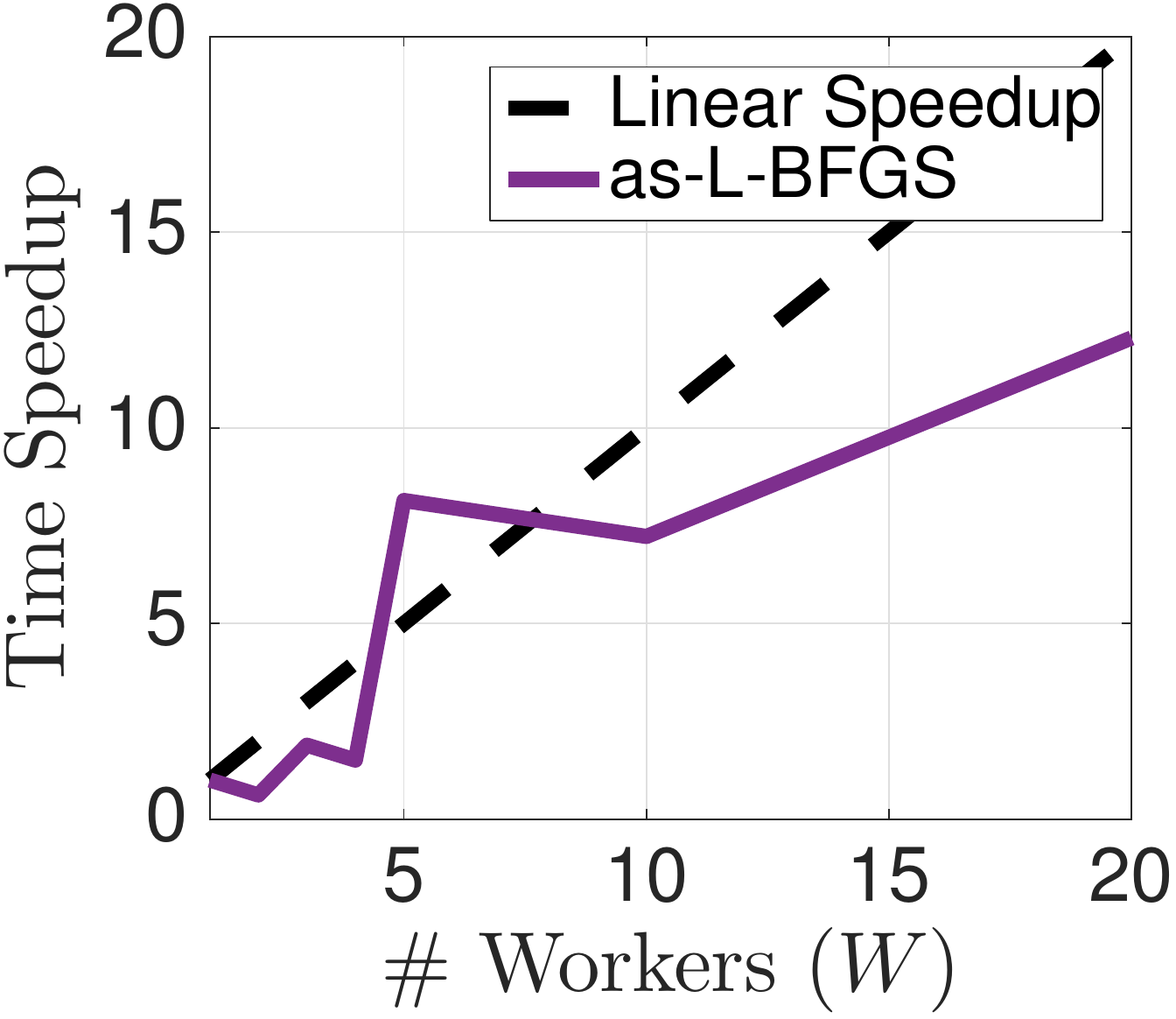}
\caption{The convergence behavior of as-L-BFGS on the ML-$1$M dataset for increasing number of workers. The `time speedup' is computed as the ratio of the running time with $1$ worker to the running time of $W$ workers.}
\label{fig:ml_speedup}
\vspace{8pt}
\end{figure}

\section{Conclusion}

In this study, we proposed an asynchronous parallel L-BFGS algorithm for non-convex optimization. We developed the algorithm within the SG-MCMC framework, where we reformulated the problem as sampling from a concentrated probability distribution.
We proved non-asymptotic guarantees and showed that as-L-BFGS achieves an ergodic global convergence with rate ${\cal O}(1/\sqrt{N})$ and it can achieve a linear speedup. Our experiments supported our theory and showed that the proposed algorithm provides a significant speedup over the synchronous parallel L-BFGS algorithm.

\section*{Acknowledgments}

The authors would like to thank to Murat A. Erdo\u{g}du for fruitful discussions.
This work is partly supported by the French National Research Agency (ANR) as a part of the FBIMATRIX project (ANR-16-CE23-0014), by the Scientific and Technological Research Council of Turkey (T\"{U}B\.{I}TAK) grant number 116E580, and by the industrial chair Machine Learning for Big Data from T\'{e}l\'{e}com ParisTech.

\bibliography{asynch_mcmc}
\bibliographystyle{icml2018}

\vfill
\pagebreak

\newpage

\onecolumn

\icmltitle{Asynchronous Stochastic Quasi-Newton MCMC for Non-Convex Optimization\\ {\normalsize SUPPLEMENTARY DOCUMENT}}

\icmltitlerunning{Asynchronous Stochastic Quasi-Newton MCMC -- Supplementary Document}


{
	\centering
	\textbf{Umut \c Sim\c sekli}$^{\text{1}}$, \textbf{\c Ca\u{g}atay Y{\i}ld{\i}z}$^{\text{2}}$, \textbf{Thanh Huy Nguyen}$^{\text{1}}$, \textbf{Ga\"{e}l Richard}$^{\text{1}}$, \textbf{A. Taylan Cemgil}$^{\text{3}}$ \\
	1: LTCI, T\'{e}l\'{e}com ParisTech, Universit\'{e} Paris-Saclay, 75013, Paris, France \\ 
	2: Department of Computer Science, Aalto University, Espoo, 02150, Finland\\ 
	3: Department of Computer Engineering, Bo\u{g}azi\c ci University, 34342, Bebek, Istanbul, Turkey 

}



\setcounter{section}{0}
\setcounter{equation}{0}
\setcounter{figure}{0}
\setcounter{table}{0}
\setcounter{page}{1}
 \renewcommand{\theequation}{S\arabic{equation}}
 \renewcommand{\thefigure}{S\arabic{figure}}
 \renewcommand{\thelemma}{S\arabic{lemma}}

 \section{The Approximate Euler-Maruyama Scheme}

\subsection{Connection with gradient descent with momentum}
\label{sec:gd_momentum}

The standard Euler-Maruyama scheme for the SDE \eqref{eqn:sde} can be developed as follows: 
\begin{align}
\theta_{n+1} &= \theta_{n} + h H_n(\theta_n) p_n, \label{eqn:euler_th1}  \\
p_{n+1} &= p_n -h H_n(\theta_n) \nabla_\theta U(\theta_n) - h \gamma p_n + \frac{h}{\beta} \Gamma_{n}(\theta_n)   + \sqrt{\frac{2h\gamma}{\beta}  } Z_{n+1} \label{eqn:euler_p1} \\
&= (1-h\gamma)p_n -h H_n(\theta_n) \nabla_\theta U(\theta_n) + \frac{h}{\beta} \Gamma_{n}(\theta_n)   + \sqrt{\frac{2h\gamma}{\beta}  } Z_{n+1} \label{eqn:euler_simple}
\end{align}
where $h$ is the step-size and $\{Z_n\}_{n=1}^N$ is a collection of standard Gaussian random variables.

We can obtain simplified update rules if we define $u_n \triangleq h p_n $ and use it in \eqref{eqn:euler_simple}. The modified update rules are given as follows:
\begin{align}
h p_{n+1} &= h p_n - h^{2} H_n(\theta_n) \nabla_\theta U(\theta_n) - h^2 \gamma p_n + \frac{h^2}{\beta} \Gamma_{n}(\theta_n)   + \sqrt{\frac{2h^3\gamma}{\beta}  } Z_{n+1} \\
u_{n+1} &= u_n - h^{2} H_n(\theta_n) \nabla_\theta U(\theta_n) - h \gamma u_n + \frac{h^2}{\beta} \Gamma_{n}(\theta_n)   + \sqrt{\frac{2h^3\gamma}{\beta}  } Z_{n+1} \\
&= \underbrace{(1- h \gamma)}_{\gamma'} u_n - \underbrace{h^{2}}_{h'} H_n(\theta_n) \nabla_\theta U(\theta_n) + \frac{h^2}{\beta} \Gamma_{n}(\theta_n)   + \sqrt{\frac{2h^3\gamma}{\beta}  } Z_{n+1} \\
&= \gamma' u_n - h' H_n(\theta_n) \nabla_\theta U(\theta_n) + \frac{h'}{\beta} \Gamma_{n}(\theta_n)   + \sqrt{\frac{2 h'(1-\gamma') }{\beta}  } Z_{n+1}, \label{eqn:u_update} 
\end{align}
where $\gamma' \in (0,1)$. If we use the modified Euler scheme as described in \citepNew{neal2010} and replace $p_n$ with $p_{n+1}$ in \eqref{eqn:euler_th1}, we obtain the following update equation:
\begin{align}
 \theta_{n+1} &= \theta_{n} + h H_n(\theta_n) p_{n+1} \\
  &= \theta_{n} + H_n(\theta_n) u_{n+1}.
 \end{align} 
Note that, when $\beta \rightarrow \infty$ we have the following update rules:
\begin{align}
u_{n+1} &= \gamma' u_n - h' H_n(\theta_n) \nabla_\theta U(\theta_n) \\
 \theta_{n+1} &=  \theta_{n} + H_n(\theta_n) u_{n+1},
\end{align}
which coincides with Gradient descent with momentum when $H_n(\theta) = I$ for all $n$.

\subsection{Numerical integration with stale stochastic gradients}

We now focus on \eqref{eqn:euler_th1} and \eqref{eqn:euler_p1}. We first drop the term $\Gamma_n$, replace the gradients $\nabla U$ with the stochastic gradients, and then modify the update rules by using stale parameters $\theta_{n-l_n}$ and $p_{n-l_n}$. The resulting scheme is given as follows:
\begin{align}
\theta_{n+1} &= \theta_{n} + h H_n(\theta_{n-l_n}) p_{n-l_n},  \label{eqn:theta update}\\
p_{n+1} &= p_n -h H_n(\theta_{n-l_n}) \nabla_\theta \tilde{U}_n(\theta_{n-l_n}) - h \gamma p_{n-l_n}   + \sqrt{\frac{2h\gamma}{\beta}  } Z_{n+1}. \label{eqn:p update}
\end{align}
By using a similar argument to the one used in Section~\ref{sec:gd_momentum}, we define $u_n \triangleq h p_n $, $h' = h^2$, $\gamma' = h \gamma$, and obtain the following update equations:
\begin{align}
\theta_{n+1} &= \theta_{n} +  H_n(\theta_{n-l_n}) u_{n-l_n}, \\
u_{n+1} &= u_n - h' H_n(\theta_{n-l_n}) \nabla_\theta \tilde{U}_n(\theta_{n-l_n}) - \gamma' u_{n-l_n}   + \sqrt{\frac{2h'\gamma'}{\beta}  } Z_{n+1}.
\end{align}

\section{Proof of Proposition~\ref{prop:inv_meas_simple}}

\begin{proof}
We start by rewriting the SDE given in \eqref{eqn:sde} as follows:
\begin{align}
d \x_t = \left\{- \left( \underbrace{\begin{bmatrix} 
		   		0 & 0  \\ 
		   		0 & \frac{\gamma}{\beta} I  
		   	\end{bmatrix}}_{\mathbf{D}}
		   	+ 
		\underbrace{\begin{bmatrix} 
		   		0 & - \frac{H_t(\theta_t)}{\beta}  \\ 
		   		\frac{H_t(\theta_t)}{\beta} & 0   
		   	\end{bmatrix}}_{\mathbf{Q}_t(\x)}
		   	\right)
		   \underbrace{	\begin{bmatrix} 
		   		\beta \nabla_\theta U(\theta_t)  \\ 
		   		\beta p_t
		   	\end{bmatrix}}_{\nabla_\x {\cal E}(\x_t)}
		   	+
		   	\underbrace{\begin{bmatrix} 
		   		0  \\ 
		   		\frac1{\beta} \Gamma_t(\theta_t)
		   	\end{bmatrix}}_{\mathbf{\Gamma}_t(\x_t)} \right\}dt +  \sqrt{2 \mathbf{D}} d W_t.
		   	\label{eqn:sde_extended}
\end{align}
Here, we observe that $\mathbf{D}$ is positive semi-definite, $\mathbf{Q}$ is anti-symmetric. Furthermore, for all $i \in \{1,2,\dots,2d\}$ we observe that
\begin{align}
\Bigl[\mathbf{\Gamma}_{t}(\x)\Bigr]_i = \sum_{j=1}^{2d} \frac{\partial [\mathbf{D} + \mathbf{Q}_t(\x)]_{ij} }{\partial_{\x_j}}.
\end{align}
The assumptions \Cref{asmp:lipschitz,asmp:H_lipschitz} directly imply that the function $-(\mathbf{D} + \mathbf{Q}_t(\x))\nabla_\x {\cal E}(\x) + \mathbf{\Gamma}_t(\x)$ is locally Lipschitz continuous in $\x$ for all $t$. Then, the desired result is obtained by applying Theorem~1 of \citepNew{ma2015complete} and Proposition~4.2.2 of \citepNew{kunze}.
\end{proof}

\section{Proof of Lemma~\ref{lem:euler}}

\subsection{Preliminaries}

In the rest of this document, if there is no specification, the notation $\mathds{E}\bigl[F\bigr]$ will denote the expectation taken over \textit{all the random sources} contained in $F$.

Before providing the proof of Lemma~\ref{lem:euler}, let us consider the following It\^{o} diffusion:
\begin{align}\label{eqn:Ito}
d\x_t=b(\x_t)dt + \sigma(\x_t)dW_t,
\end{align}
where $\x_t\in\mathds{R}^{2d}$, $b:\,\mathds{R}^{2d}\rightarrow\mathds{R}^{2d}$, $\sigma:\,\mathds{R}^{2d}\rightarrow\mathds{R}^{2d\times 2d}$, and $W_t$ is Brownian motion in $\mathds{R}^{2d}$. The generator $\Lo$ for \eqref{eqn:Ito} is formally defined as follows:
\begin{align}\label{eqn:Ito Generator}
\Lo f(\x_t)\triangleq\lim_{h\rightarrow 0^+}\frac{\mathds{E}[f(\x_{t+h})\vert\x_t]-f(\x_t)}{h}=\Big(b(\x_t)\cdot\nabla+\frac{1}{2}\big(\sigma(\x_t)\sigma(\x_t)^\top\big):\nabla\nabla^\top\Big)f(\x_t),
\end{align}
where $f:\mathds{R}^n\rightarrow\mathds{R}$ is any twice differentiable function whose support is compact. Here, $a\cdot b$ denotes the inner product between vectors $a$ and $b$, $A:B$ by definition is equal to $tr\{A^{\top}B\}$ for some matrices $A$ and $B$. In our study, the generator for the diffusion \eqref{eqn:sde_extended} is then defined as follows: (define $n = t/h$ and use \eqref{eqn:Ito Generator})
\begin{align}
\mathcal{L}_n\triangleq\Big(H_n p_n\cdot\nabla_{\theta} - \big(H_n\nabla_{\theta}U(\theta_{n}) + \gamma p_n - \frac{1}{\beta}\Gamma_n(\theta_{n-\stl_n})\big)\cdot\nabla_p\Big) + \mathbf{D}:\nabla_{\x}\nabla_{\x}^{\top}.
\end{align}
We also define the following operator for the approximate Euler-Maruyama scheme with delayed updates:
\begin{align}\label{eqn:localGen}
\tilde{\mathcal{L}}_n\triangleq\Big(H_n p_n\cdot\nabla_{\theta} - \big(H_n \nabla_{\theta}\tilde{U}(\theta_{n-\stl_n}) + \gamma p_n\big) \cdot\nabla_p\Big) + \mathbf{D}:\nabla_{\x}\nabla_{\x}^{\top}.
\end{align}
By using the definitions $\Lo_n$ and $\Lot_n$, we obtain the following identity:
\begin{align}
\Lot_n = \Lo_n - \Delta V_n, \label{eqn:op_approx}
\end{align}
where $\Delta V_n \triangleq \Bigl( H_{n}(\theta_{n-\stl_n}) (\nabla_\theta \tilde{U}_{n}(\theta_{n-\stl_n}) - \nabla_\theta U(\theta_n)) + \frac1{\beta} \Gamma(\theta_{n-l_n}) \Bigr) \cdot \nabla_p$. This operator essentially captures all the errors induced by the approximate integrator.

We now proceed to the proof of Lemma~\ref{lem:euler}. The proof uses several technical lemmas that are given in Section~\ref{sec:techlemma}.

\subsection{Proof of Lemma~\ref{lem:euler}}

\begin{proof}
We first consider the Euler-Maruyama integrator of the SDE \eqref{eqn:sde_extended}, to combine \eqref{eqn:euler_th1} and \eqref{eqn:euler_simple} into a single equation, given as follows:
\begin{align*}
\x_{n+1} = \x_{n} - h (\mathbf{D} + \mathbf{Q}_n(\x_n)) \nabla{\cal E}(\x_{n}) + h \mathbf{\Gamma}_{n+1}(\x_{n}) + \sqrt{2 h\mathbf{D}} Z'_{n+1}
\end{align*}
where $Z'_n$ is a standard Gaussian random variable in $\mathds{R}^{2d}$, h is the step-size, $\mathbf{D}$, $\mathbf{Q}$, and $\mathbf{\Gamma}$ are defined in \eqref{eqn:sde_extended}. We then modify this scheme such that we replace the gradient $\nabla {\cal E}$ with the stale stochastic gradients and we discard the term $\mathbf{\Gamma}$. The resulting numerical integrator is given as follows:
\begin{align}
\x_{n+1} = \x_{n} - h (\mathbf{D} + \mathbf{Q}_n(\x_{n-l_n})) \nabla\tilde{\cal E}_n(\x_{n-l_n}) + \sqrt{2 h\mathbf{D}} Z'_{n+1}. \label{eqn:euler_x_final}
\end{align}
Note that \eqref{eqn:euler_x_final} coincides with the proposed algorithm, given in \eqref{eqn:update_th_ult}.

In the sequel, we follow a similar strategy to \citepNew{chen2016stochastic}. However, we have additional difficulties caused by the usage of L-BFGS matrices, which are reflected in the operator $\Delta V_n$.
Since we are using the Euler-Maruyama integrator, we have the following inequality \citepNew{chen2015convergence}: 
\begin{align}
\E[\psi(\x_n)\vert\x_{n-1}] = (\mathds{I} + h \Lot_n) \psi(\x_{n-1}) + \Oc(h^2).
\label{eqn:euler_approx}
\end{align}
By summing both sides of \eqref{eqn:euler_approx} over $n$, taking the expectation, and using \eqref{eqn:op_approx}, we obtain the following: 
\begin{align}
\sum_{n=1}^N \E[\psi(\x_n)] = \psi(\x_0) + \sum_{n=1}^{N-1} \E[\psi(\x_n)] - h \sum_{n=1}^N \E[\Delta V_n \psi(\x_{n-1})] + h \sum_{n=1}^N \E[\Lo_n \psi(\x_{n-1})] + \Oc(Nh^2).
\end{align}
By rearranging the terms and dividing all the terms by $Nh$, we obtain:
\begin{align}
\frac{\E{\psi(\x_N)} - \psi(\x_0)}{Nh} = \frac{-\sum_{n=1}^N \E[\Delta V_n \psi(\x_{n-1})] + \sum_{n=1}^N \E[\Lo_n \psi(\x_{n-1})]}{N} + \Oc(h).
\end{align}
By using the Poisson equation given in \eqref{eqn:poisson_eq} for each $\Lo_n \psi(\x_{n-1})$ and rearranging the terms, we obtain:
\begin{align}
\E[\frac1{N} \sum_n(U(\theta_n) - \bar{U}_{\beta}) ] = \frac{\E[\psi(\x_N)] - \psi(\x_0)}{Nh} + \frac{\sum_{n=1}^N \E[\Delta V_n \psi(\x_{n-1})]}{N} + \Oc(h).
\end{align}
By assumption \Cref{asmp:poisson}, the term $\E[\psi(\x_N)] - \psi(\x_0)$ is uniformly bounded. Then, by Assumption \Cref{asmp:poisson} and Lemma~\ref{cor:deltav}, we obtain the following bound:
\begin{align}
\E[\frac1{N} \sum_n(U(\theta_n) - \bar{U}_{\beta}) ] = \Oc \Bigl( \frac{1}{Nh} + \max(\stl_{\text{max}},1)h + \frac{1}{\beta}\Bigr), \label{eqn:bound_bias_interm}
\end{align}
as desired.
\end{proof}

\begin{remark}
Theorem~\ref{thm:bias} significantly differentiates from other recent results. First of all, none of the references we are aware of provides an analysis for an asynchronous stochastic L-BFGS algorithm. Aside from this fact, when compared to \citepNew{chen2016bridging}, our bound handles the case of delayed updates and provides an explicit dependence on $\beta$. When compared to \citepNew{chen2016stochastic}, our analysis considers the tempered case and handles the additional difficulties brought by the L-BFGS matrices and their derivatives. On the other hand, our analysis is also significantly different than the ones presented in \citepNew{raginsky17a} and \citepNew{xu2017global}, as it considers the asynchrony and L-BFGS matrices, and provides a bound for the ergodic error.
\end{remark}

\section{Proof of Lemma~\ref{lem:entropy}}

\begin{proof}
We use the same proof technique given in \citepNew{raginsky17a}[Proposition 11]. We assume that $\pi_\theta$ admits a density with respect to the Lebesgue measure, denoted as $\rho(\theta) \triangleq \frac1{Z_\beta} \exp(-\beta U(\theta))$, where $Z_\beta$ is the normalization constant: $Z_\beta \triangleq \int_{\mathds{R}^d} \exp(-\beta U(\theta)) d\theta$. We start by using the definition of $\bar{U}_\beta$, as follows:
\begin{align}
	\bar{U}_\beta = \int_{\mathds{R}^d} U(\theta) \pi_\theta(d\theta) = \frac1{\beta} ({\cal H}(\rho) - \log Z_\beta ), \label{eqn:ubeta_def}
\end{align}
where ${\cal H}(\rho)$ is the \emph{differential entropy}, defined as follows:
\begin{align}
{\cal H}(\rho) \triangleq - \int_{\mathds{R}^d} \rho(\theta) \log \rho(\theta) d\theta.
\end{align}
We now aim at upper-bounding ${\cal H}(\rho)$ and lower-bounding $\log Z_\beta$. By Assumption \Cref{asmp:var}, the distribution $\pi_\theta$ has a finite second order moment, therefore its differential entropy is upper-bounded by the differential entropy of a Gaussian distribution that has the same second order moment. Then, we obtain
\begin{align}
{\cal H}(\rho) &\leq \frac1{2} \log [ (2\pi e)^d \det( \Sigma)] \\
&\leq \frac1{2} \log [ (2\pi e)^d \Bigl(\frac{\text{tr} (\Sigma)}{d}\Bigr)^d] \label{eqn:entropy1} \\
&\leq \frac{d}{2} \log \Bigl( 2\pi e \frac{C_\beta}{\beta d}\Bigr), \label{eqn:entropy2}
\end{align}
where $\Sigma$ denotes the covariance matrix of the Gaussian distribution. In \eqref{eqn:entropy1} we used the relation between the arithmetic and geometric means, and in \eqref{eqn:entropy2} we used Assumption \Cref{asmp:var}. 

We now lower-bound $\log Z_\beta$. By definition, we have
\begin{align}
\log Z_\beta &= \log \int_{\mathds{R}^d} \exp(-\beta U(\theta)) d \theta \\
&= -\beta U^\star  + \log \int_{\mathds{R}^d} \exp(\beta (U^\star - U(\theta))) d \theta \\
&\geq -\beta U^\star  + \log \int_{\mathds{R}^d} \exp(-\frac{\beta L \|\theta - \theta^\star\|^2}{2} ) d \theta \label{eqn:laplace_inter} \\
&= -\beta U^\star  + \frac{d}{2} \log(\frac{2\pi}{L \beta}). \label{eqn:entropy_bound2}
\end{align}
Here, in \eqref{eqn:laplace_inter} we used Assumption \Cref{asmp:lipschitz} and \citepNew{nesterov2013introductory}(Lemma 1.2.3).

Finally, by combining \eqref{eqn:ubeta_def}, \eqref{eqn:entropy2}, and \eqref{eqn:entropy_bound2}, we obtain
\begin{align}
\bar{U}_\beta - U^\star &= \frac1{\beta}({\cal H}(\rho) - \log Z_\beta ) - U^\star \\ 
& \leq \frac1{\beta}\Bigl( \frac{d}{2} \log \bigl( 2\pi e \frac{C_\beta}{\beta d}\bigr) +\beta U^\star  - \frac{d}{2} \log(\frac{2\pi}{L \beta}) \Bigr) - U^\star \\
&= \frac1{\beta} \frac{d}{2} \log\Bigl( \frac{ e C_\beta L }{ d}  \Bigr) \\
&= {\cal O}\Bigl( \frac1{\beta} \Bigr).
\vspace{-10pt}
\end{align}
This finalizes the proof.
\end{proof}

\section{Proof of Theorem~\ref{thm:bias}}

\begin{proof}

We decompose the error, as follows:
\vspace{-10pt}
\begin{align}
\Bigl| \mathds{E}\hat{U}_N - U^\star \Bigr| &= \Bigl| \mathds{E}[\frac1{N} \sum_{n=1}^N(U(\theta_n) - U^\star)] \Bigr|\\
&= \Bigl|  \mathds{E}[ \frac1{N} \sum_{n=1}^N \bigl( U(\theta_n) - \bar{U}_{\beta}] \bigr) + \bigl(\bar{U}_{\beta} - U^\star \bigr) \Bigr|  \\
&\leq \underbrace{\Bigl|  \mathds{E}[ \frac1{N} \sum_{n=1}^N \bigl( U(\theta_n) - \bar{U}_{\beta}] \bigr) \Bigr|}_{{\cal A}_1} + \underbrace{\Bigl( \bar{U}_{\beta} - U^\star  \Bigr)}_{{\cal A}_2},
\label{eqn:bound_bias_pre}
\end{align}
where the term ${\cal A}_1$ is bounded by Lemma~\ref{lem:euler} and the term ${\cal A}_2$ is bounded by Lemma~\ref{lem:entropy}. This finalizes the proof. 
\end{proof}

\section{Technical Lemmas}

\label{sec:techlemma}

For convenience, let us introduce the following notations:
$\bar{\x}_k\triangleq(\x_0,\ldots,\x_k)$. 
Let us also denote $\Omega_n$ the (uniform) random subsample, which is chosen independently of ($\bar{\x}_n$), used for iteration $n$.
\begin{lemma}
\label{lemma:expectation on X}
Let $f_k (\x)\triangleq \| \x - \x_{k-1} \|$. Under the assumptions \Cref{asmp:H_lipschitz}-\ref{asmp:Taylor expansion}, the following bound holds: 
\begin{align}
\label{eqn:expectation on gradient}
\mathds{E}_{\bar{\x}_n} \bigl[\| \nabla_\theta U(\theta_{n-\stl_n}) - \nabla_\theta U(\theta_{n}) \|\bigr] =   {\cal O}\Bigl(\stl_{\text{max}} h \max_{i \in \llbracket n-\stl_{\text{max}} + 1, n \rrbracket} \mathds{E}\bigl[\Lo_i f_i(\x_{i-1})\bigr] + h^2\Bigr)
\end{align}
where $\mathds{E}_{\bar{\x}_k}$ denotes the expectation taken over the random variables $\x_0,\ldots,\x_k$.
\end{lemma}

\begin{proof}
The proof is similar to [\citepNew{chen2016stochastic}, Lemma 8], we provide the proof for completeness. We first consider the following estimate which uses the Lipschitz property of $\nabla_\theta U(\theta)$:
\begin{align}
\nonumber
\mathds{E}_{\bar{\x}_n} \bigl[\| \nabla_\theta U(\theta_{n-\stl_n}) - \nabla_\theta U(\theta_{n}) \|\bigr] &\leq L\mathds{E}_{\bar{\x}_n} \bigl[\|\theta_{n-\stl_n} - \theta_{n}\|\bigr]\\
\nonumber&\leq L\mathds{E}_{\bar{\x}_n} \Bigl[\Big\|\sum_{i=n-\stl_n}^{n-1}(\theta_{i} - \theta_{i+1})\Big\|\Bigr]\\
\nonumber&\leq L\sum_{i=n-\stl_{n}}^{n-1}\mathds{E}_{\bar{\x}_n} \Bigl[\Big\|\theta_{i} - \theta_{i+1}\Big\|\Bigr]\\
\nonumber&\leq L\sum_{i=n-\stl_{n}}^{n-1}\mathds{E}_{\bar{\x}_n} \Bigl[\Big\|\x_{i} - \x_{i+1}\Big\|\Bigr]\\
&= L\sum_{i=n-\stl_{n}}^{n-1}\mathds{E}_{\bar{\x}_n} \bigl[f_{i+1} (\x_{i+1})\bigr]. \label{eqn:estimate of expectation}
\end{align}

Using law of total expectation, we have 
\begin{align}
\label{eqn:estimate of Euler integerator}
\nonumber \mathds{E}_{\bar{\x}_n}\bigl[f_{i+1} (\x_{i+1})\bigr]&=\mathds{E}\bigl[f_{i+1} (\x_{i+1})\bigr]\\
\nonumber&=\mathds{E}\bigl[\mathds{E}\bigl[f_{i+1} (\x_{i+1})\vert \x_i\bigr]\bigr]\\
\nonumber&=\mathds{E}\bigl[e^{h\Lo_{i+1}}f_{i+1}(\x_{i}) + {\cal O}(h^2)\bigr]\\
\nonumber&=\mathds{E}\bigl[f_{i+1}(\x_{i}) + h\Lo_{i+1}f_{i+1}(\x_{i}) + {\cal O}(h^2)\bigr]\\
&\leq h\mathds{E}\bigl[\Lo_{i+1} f_{i+1}(\x_{i})\bigr] + {\cal O}(h^2).
\end{align}
The third equality is due to the fact that Euler integrator is a first order integrator. Then we applied Assumption \Cref{asmp:Taylor expansion} and $f_{i+1}(\x_{i}) = 0$ to obtain the last two lines. Finally, by combining \eqref{eqn:estimate of expectation} and \eqref{eqn:estimate of Euler integerator}, we obtain:
\begin{align*}
\mathds{E}_{\bar{\x}_n} \bigl[\| \nabla_\theta U(\theta_{n-\stl_n}) - \nabla_\theta U(\theta_{n}) \|\bigr] &\leq L\sum_{i=n-\stl_{n}}^{n-1}(h\mathds{E}\bigl[\Lo_{i+1} f_{i+1}(\x_{i})\bigr] + {\cal O}(h^2))\\
&\leq L\sum_{i=n-\stl_{\text{max} }}^{n-1}(h\mathds{E}\bigl[\Lo_{i+1} f_{i+1}(\x_{i})\bigr] + {\cal O}(h^2))\\
&\leq L \stl_{\text{max}} h \max_{i \in \llbracket n-\stl_{\text{max}} + 1, n \rrbracket} \mathds{E}\bigl[\Lo_{i} f_{i}(\x_{i-1})\bigr] + {\cal O}(h^2).
\end{align*}
This completes the proof.
\end{proof}

\begin{lemma}
\label{lemma:gamma}
If Assumption \Cref{asmp:H_lipschitz} holds then the following bound holds: 
\begin{align}
\|\mathbf{\Gamma}_n\| = {\cal O}(\frac1{\beta}),
\end{align}
where $\mathbf{\Gamma}_n$ is defined in \eqref{eqn:sde_extended}.
\end{lemma} 

\begin{proof}
If $\stl_n > 0$ then $\| \Gamma_n(\theta_n) \| = 0$ since $H_n$ will not depend on $\theta_n$ (see \eqref{eqn:gamma_term} for the definition of $\Gamma_n$). For $\stl_n = 0$, by the Lipschitz continuity of $H_n$, the first order partial derivatives of $H_n$ are all bounded by $L_H$. Then, $ \| \mathbf{\Gamma}_n\| = \frac1{\beta}\|\Gamma_n\|$ is therefore bounded by a quantity that is proportional to $\beta^{-1}$.
\end{proof}

\begin{lemma}
\label{cor:deltav}
Let $f_k (\x)\triangleq \| \x - \x_{k-1} \|$. Under the assumptions \Cref{asmp:H_lipschitz}-\ref{asmp:Taylor expansion}, the following bound holds:
\begin{align}
\label{eqn:bound delta V}
\mathds{E}\bigl[\Delta V_n\psi(\x_{n-1})\bigr] = {\cal O}\Bigl(\stl_{\text{max}} h \max_{i \in \llbracket n-\stl_{\text{max}} + 1, n \rrbracket} \mathds{E}\bigl[\Lo_{i} f_{i}(\x_{i-1})\bigr] + h^2 + \beta^{-1}\Bigr).
\end{align}
\end{lemma}

\begin{proof}
First, by using the triangular inequality we have:
\begin{align}
\label{eqn:main ineq}
\nonumber \Vert\mathds{E}\bigl[\Delta V_n\psi(\x_{n-1})\bigr]\Vert =& \Vert\mathds{E}\bigl[( H_{n}(\theta_{n-\stl_n})(\nabla_\theta \tilde{U}_{n-\stl_n}(\theta_{n-\stl_n}) - \nabla_\theta U(\theta_{n})) + \frac{1}{\beta}\Gamma_n(\theta_{n-\stl_n}))\cdot\nabla_p \psi(\x_{n-1})\bigr]\Vert\\
\nonumber \leq& \Vert\mathds{E}\bigl[ H_{n}(\theta_{n-\stl_n})(\nabla_\theta \tilde{U}_{n-\stl_n}(\theta_{n-\stl_n}) - \nabla_\theta U(\theta_{n}))\cdot\nabla_p \psi(\x_{n-1})\bigr]\Vert \\
&+ \Vert\mathds{E}\bigl[\frac{1}{\beta}\Gamma_n(\theta_{n-\stl_n})\cdot\nabla_p \psi(\x_{n-1})\bigr]\Vert .
\end{align}
Applying Assumption \Cref{asmp:poisson} and Lemma \ref{lemma:gamma}, we obtain the bound for the second term in the above sum:
\begin{align}
\label{eqn:sub ineq}
A_1\triangleq\Vert\mathds{E}\bigl[\frac{1}{\beta}\Gamma_n(\theta_{n-\stl_n})\cdot\nabla_p \psi(\x_{n-1})\bigr]\Vert = {\cal O}(\beta^{-1}).
\end{align}
We note that the expectation is taken over $(\bar{\x}_n,\Omega_n)$, where $\bar{\x}_n$ and $\Omega_n$ are independent. Hence, the first term in \eqref{eqn:main ineq} can be rewritten as follows:
\begin{align*}
A_2 \triangleq& \Vert\mathds{E}\bigl[ H_{n}(\theta_{n-\stl_n})(\nabla_\theta \tilde{U}_{n-\stl_n}(\theta_{n-\stl_n}) - \nabla_\theta U(\theta_{n}))\cdot\nabla_p \psi(\x_{n-1})\bigr]\Vert\\
=& \Vert\mathds{E}_{\bar{\x}_n}\bigl[\mathds{E}_{\Omega_n}\bigl[ H_{n}(\theta_{n-\stl_n})(\nabla_\theta \tilde{U}_{n-\stl_n}(\theta_{n-\stl_n}) - \nabla_\theta U(\theta_{n}))\cdot\nabla_p \psi(\x_{n-1})\bigr]\bigr]\Vert\\
=&\big\Vert\mathds{E}_{\bar{\x}_n}\Bigl[\mathds{E}_{\Omega_n}\bigl[ H_{n}(\theta_{n-\stl_n})(\nabla_\theta \tilde{U}_{n-\stl_n}(\theta_{n-\stl_n}) - \nabla_\theta \tilde{U}_{n-\stl_n}(\theta_{n}))\cdot\nabla_p \psi(\x_{n-1})\bigr] + \mathds{E}_{\Omega_n}\bigl[ H_{n}(\theta_{n-\stl_n})(\nabla_\theta \tilde{U}_{n-\stl_n}(\theta_{n}) 
\\&- \nabla_\theta U_{n}(\theta_{n}))\cdot\nabla_p \psi(\x_{n-1})\bigr] \Bigr]\big\Vert.
\end{align*}
As $H_n$ and $\nabla_p \psi(\x_{n-1})$ are independent of the random subsample $\Omega_n$, we have
\begin{align*}
\mathds{E}_{\Omega_n}\bigl[ H_{n}(\theta_{n-\stl_n})(\nabla_\theta \tilde{U}_{n-\stl_n}(\theta_{n}) - \nabla_\theta U_{n}(\theta_{n}))\cdot\nabla_p \psi(\x_{n-1})\bigr] &= H_{n}(\theta_{n-\stl_n})\mathds{E}_{\Omega_n}\bigl[ \nabla_\theta \tilde{U}_{n-\stl_n}(\theta_{n}) - \nabla_\theta U_{n}(\theta_{n})\bigr]\cdot\nabla_p \psi(\x_{n-1}) \\&= 0.
\end{align*}
As a result,
\begin{align}
\label{eqn:sub ineq 2}
\nonumber A_2 &=\Vert\mathds{E}_{\bar{\x}_n}\bigl[\mathds{E}_{\Omega_n}\bigl[ H_{n}(\theta_{n-\stl_n})(\nabla_\theta \tilde{U}_{n-\stl_n}(\theta_{n-\stl_n}) - \nabla_\theta \tilde{U}_{n-\stl_n}(\theta_{n}))\cdot\nabla_p \psi(\x_{n-1})\bigr]\bigr]\|\\
\nonumber &=\Vert\mathds{E}_{\bar{\x}_n}\bigl[ H_{n}(\theta_{n-\stl_n})(\nabla_\theta U(\theta_{n-\stl_n}) - \nabla_\theta U(\theta_{n}))\cdot\nabla_p \psi(\x_{n-1})\bigr]\|\\
\nonumber &\leq C\mathds{E}_{\bar{\x}_n}\bigl[ \|\nabla_\theta U(\theta_{n-\stl_n}) - \nabla_\theta U(\theta_{n})\|\bigr]\\
&={\cal O}\Bigl(\stl_{\text{max}} h \max_{i \in \llbracket n-\stl_{\text{max}} + 1, n \rrbracket} \mathds{E}\bigl[\Lo_{i} f_{i}(\x_{i-1})\bigr] + h^2\Bigr).
\end{align}
The inequality in \eqref{eqn:sub ineq 2} is deduced from the fact that $H_n$ is bounded by \citepNew{berahas2016multi}[Lemma3.3] and $\nabla_p \psi(\x_{n-1})$ is bounded by assumptions, and the last equality is due to Lemma \ref{lemma:expectation on X}. Finally, by combining \eqref{eqn:main ineq}, \eqref{eqn:sub ineq}, and \eqref{eqn:sub ineq 2}, we obtain \eqref{eqn:bound delta V}, which concludes the proof.
\end{proof}

\section{Additional Experimental Results}

In this section, we provide the result where we illustrate the iteration speedup of as-L-BFGS on the ML-$1$M dataset.

\begin{figure}[H]
\centering
\includegraphics[width=0.3\columnwidth]{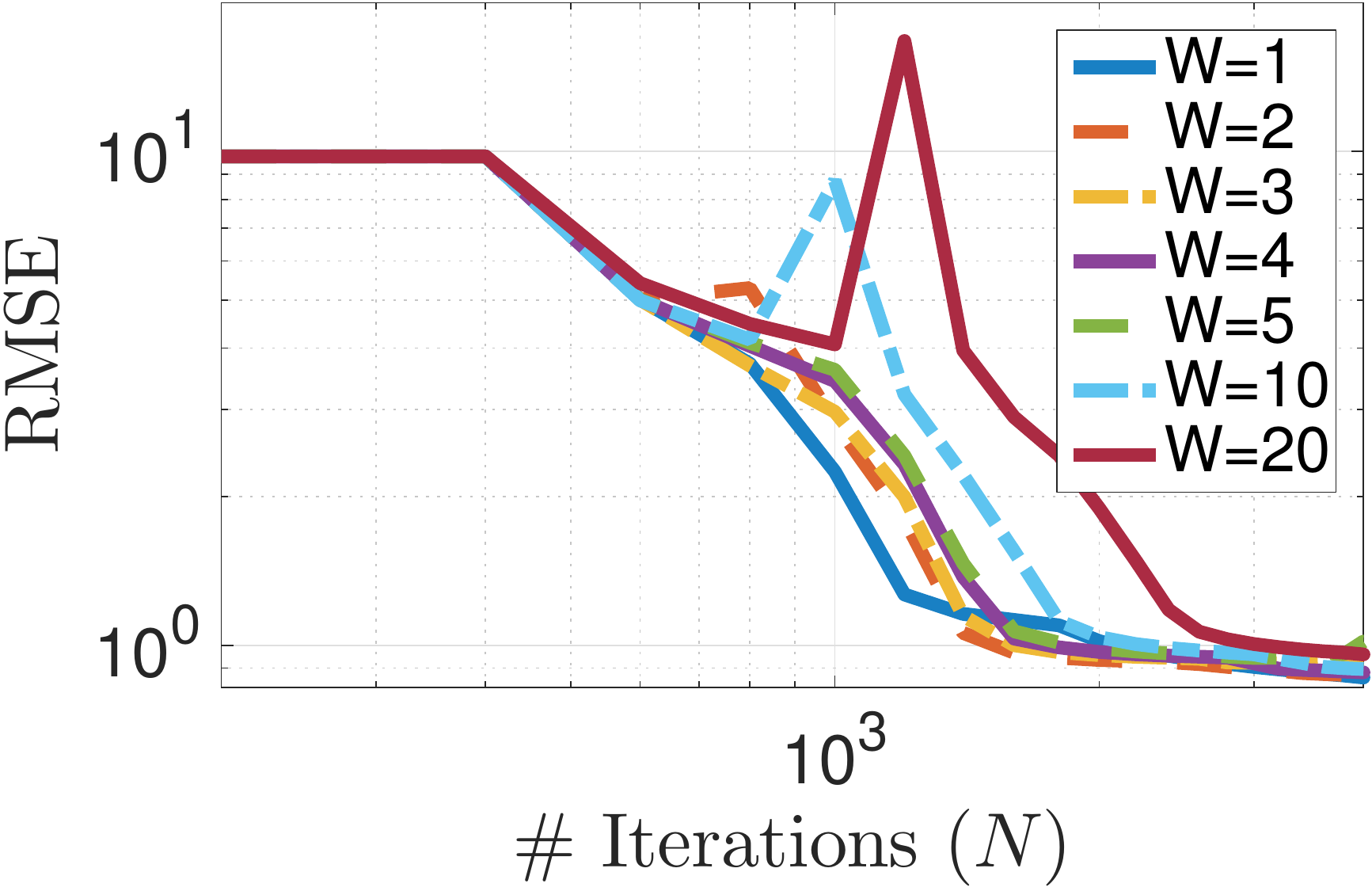}
\caption{The convergence behavior of as-L-BFGS on the ML-$1$M dataset for increasing number of workers.}
\label{fig:ml_speedup_iter}
\end{figure}

\section{Algorithm Parameters Used in the Experiments}

\subsection{Linear Gaussian model}

Table~\ref{tab:t1} lists the algorithm parameters for the synthetic data experiments. We fixed the L-BFGS memory sizes for {mb-L-BFGS} and {as-L-BFGS} to $M=3$. 
The remaining parameters are the step sizes ($h$, $h'$), timeout duration of {mb-L-BFGS} server ($T_\text{mb}$), the friction parameter ($\gamma'$), and the inverse temperature ($\beta$) of {as-L-BFGS}.

\begin{table}[h!]
	\caption{The list of algorithm parameters that are used in the experiments on the linear Gaussian model.}
	\begin{center}
		\scalebox{1}{
			\begin{tabular}{c | c | c | c| c | c |c}
				{a-SGD} & \multicolumn{2}{c |}{{mb-L-BFGS}} & \multicolumn{3}{c}{{as-L-BFGS}} \\ 
				\toprule
				$h$ & $h$ & $T_\text{mb}$ (base units) & $h'$ & $\gamma'$ & $\beta$\\ 
				\midrule
				 $1\times 10^{-3}$ & $5 \times 10^{-2}$ & $10$& $4\times 10^{-4}$ & $3 \times 10^{-2}$ & $5\times 10^{2}$  \\ 
				\bottomrule
			\end{tabular}
		}
	\end{center}
	\label{tab:t1}
\end{table}

Table~\ref{tab:t3} lists the parameters of the simulator. The parameters are (i) $\mu_m$: the average computational time spent by the master node at each iteration, (ii) $\mu_w$: the average computational time spent by a single worker at each iteration, and (iii) $\tau$: the time spent for communication per iteration. In all cases we set $\tau = 10$, $N_\Omega = N_Y/100$, $N_O = N_\Omega/3$.

\begin{table}[h!]
	\caption{The list of simulator parameters that are used in the experiments on the linear Gaussian model.}
	\begin{center}
		\scalebox{1}{
			\begin{tabular}{c | c | c | c| c | c  }
				\multicolumn{2}{c |}{{a-SGD}} & \multicolumn{2}{c |}{{mb-L-BFGS}} & \multicolumn{2}{c}{{as-L-BFGS}} \\ 
				\toprule
				$\mu_m$ & $\mu_w$  & $\mu_m$ & $\mu_w$  &  $\mu_m$ & $\mu_w$   \\ 
				\midrule
				$0$ & $1000 \times \frac{N_\Omega}{N}$  &  $30$ & $1000 \times \frac{N_\Omega}{N_Y}$ & $0$ &$1000 \times \frac{N_\Omega}{N_Y} + 60$ \\
				\bottomrule
			\end{tabular}
		}
	\end{center}
	\label{tab:t3}
\end{table}

\subsection{Large-scale matrix factorization}

Table~\ref{tab:t2} lists the algorithm parameters for different data sets. We fixed the L-BFGS memory sizes for {mb-L-BFGS} and {as-L-BFGS} to $M=3$. 
In all experiments we set $\rho =3$, $N_\Omega = N_Y/100$, $N_O = N_\Omega/3$.

\begin{table}[h!]
	\caption{The list of algorithm parameters that are used in the experiments on the large scale matrix factorization.}
	\begin{center}
		\scalebox{1}{
			\begin{tabular}{ l |c | c | c | c| c | c |c}
				& {a-SGD} & \multicolumn{2}{c |}{{mb-L-BFGS}} & \multicolumn{3}{c}{{as-L-BFGS}} \\ 
				\toprule
				& $h$ & $h$ & $T_\text{mb}$ (m. sec.) & $h'$ & $\gamma'$ & $\beta$\\ 
				\midrule
				ML-$1$M  & $1\times 10^{-6}$ & $5 \times 10^{-7}$ & $400$& $2\times 10^{-8}$ & $1 \times 10^{-1}$ & $1\times 10^{3}$   \\
				ML-$10$M & $2\times 10^{-7}$ & $1 \times 10^{-8}$ & $3400$& $1\times 10^{-9}$ & $3 \times 10^{-2}$ & $1\times 10^{3}$  \\  
				ML-$20$M & $1\times 10^{-7}$ & $1 \times 10^{-8}$ & $4500$& $1\times 10^{-9}$ & $1 \times 10^{-3}$ & $1\times 10^{3}$ \\  
				\bottomrule
			\end{tabular}
		}
	\end{center}
	\label{tab:t2}
\end{table}

\bibliographyNew{asynch_mcmc}
\bibliographystyleNew{icml2017}

\end{document}